\documentclass[onefignum,onetabnum]{siamonline250211}

\usepackage{lipsum}
\usepackage{amsfonts}
\usepackage{graphicx}
\usepackage{epstopdf}
\usepackage{algorithmic}
\usepackage{dsfont}

\ifpdf
  \DeclareGraphicsExtensions{.eps,.pdf,.png,.jpg}
\else
  \DeclareGraphicsExtensions{.eps}
\fi

\usepackage{enumitem}
\setlist[enumerate]{leftmargin=.5in}
\setlist[itemize]{leftmargin=.5in}

\newsiamremark{remark}{Remark}
\newsiamremark{hypothesis}{Hypothesis}
\crefname{hypothesis}{Hypothesis}{Hypotheses}
\newsiamthm{claim}{Claim}
\newsiamremark{fact}{Fact}
\crefname{fact}{Fact}{Facts}

\usepackage{graphicx} 
\usepackage{hyperref}
\usepackage{mathrsfs}
\usepackage{url}
\usepackage{bm}
\usepackage{mathtools,amsmath,amssymb}

\ifpdf
\hypersetup{
  pdftitle={General MDPs with Unbounded Costs and Policy Gradient Methods},
  pdfauthor={A. Gupta and A. Mahajan}
}
\fi

\headers{General MDPs with Unbounded Cost}{A. Gupta and A. Mahajan}

\title{Operator-Theoretic Foundations and Policy Gradient Methods for General MDPs with Unbounded Costs}

\author{Abhishek Gupta\thanks{The Ohio State University, Columbus, Ohio. 
  (\email{gupta.706@osu.edu}, \url{https://gupta706.github.io/}).}
\and Aditya Mahajan\thanks{McGill University, Montreal, Canada. 
  (\email{aditya.mahajan@mcgill.ca}, \url{https://adityam.github.io/})}
  }

\usepackage{amsopn}

\usepackage{enumitem}

\newcommand{\textttup}[1]{\textup{\texttt{#1}}}
\newcommand{\Na}{\mathbb{N}}
\renewcommand{\Re}{\mathbb{R}}
\newcommand{\op}{\bm}
\newcommand{\fs}{\mathcal F_{\mathcal S}}
\newcommand{\fsa}{\mathcal F_{\mathcal K}}
\newcommand{\ms}{\mathcal M_{\mathcal S}}
\newcommand{\msa}{\mathcal M_{\mathcal K}}
\newcommand{\spec}{\textttup{spec}}

\newcommand{\lifs}{\mathcal V}
\newcommand{\lifsa}{\mathcal Q}
\newcommand{\lifsb}{\mathcal V_{\textttup{b}}}
\newcommand{\lifsab}{\mathcal Q_{\textttup{b}}}

\newcommand{\absconv}{\texttt{absconv}}

\newcommand{\supp}{\texttt{supp}}
\newcommand{\id}[1]{\textttup{id}_{#1}}

\newcommand{\ids}{\textttup{id}_{\fs}}
\newcommand{\ws}{\textttup{w}_{\mathcal S}}
\newcommand{\wsa}{\textttup{w}_{\mathcal K}}
\newcommand{\metric}{\textttup{IPM}_{\Pi,\ws}}

\newcommand{\kp}{\kappa_{\op P}}
\newcommand{\pair}[1]{\left\langle #1 \right\rangle}

\newcommand{\gateaux}{\mathfrak D}

\newcommand{\stable}{\Pi_{\textup{\texttt{stable}}}}

\newcommand{\decay}{\Pi_{\textup{\texttt{decay}}}}
\newcommand{\specstable}{\Pi_{\textup{\texttt{ss}}}}

\newcommand{\exo}[2]{\mathop{\mathbb{E}}_{#2}\left[#1\right]}
\newcommand{\pr}[1]{\mathbb{P}\left\{#1\right\}}

\newcommand{\ipm}{\textttup{IPM}}
\newcommand{\gen}[1]{\mathfrak{F}_{#1}}

\newcommand{\kernels}{\mathfrak{K}_{\mathcal S}}
\newcommand{\kernelsa}{\mathfrak{K}_{\mathcal K}}

\newcommand{\ipms}{\ipm_{\gen{\lifs}}}
\newcommand{\ipmsa}{\ipm_{\gen{\lifsa}}}
\newcommand{\ipmas}[1]{\ipm_{\gen{\mathcal A}(s)} }

\DeclareMathOperator*{\argmin}{\arg\min}

\newcommand*\TRANS{{\mathpalette\doTRANS\empty}}
\makeatletter
\newcommand*\doTRANS[2]{\raisebox{\depth}{$\m@th#1\intercal$}}
\makeatother

\newtheorem{assumption}{Assumption}
\newtheorem{example}{Example}

\newcommand{\examplecontref}{}

\newtheorem*{examplecontinner}{Example \examplecontref\ (continued)}

\newenvironment{examplecont}[1]
  {\renewcommand{\examplecontref}{\ref{#1}}%
   \begin{examplecontinner}}
  {\end{examplecontinner}}

\begin{document}

\maketitle
\begin{abstract}
Markov decision processes (MDPs) is viewed as an optimization of an objective function over certain linear operators over general function spaces. A new existence result is established for the existence of optimal policies in general MDPs, which differs from the existence result derived previously in the literature. Using the well-established perturbation theory of linear operators, policy difference lemma is established for general MDPs and the Gauteaux derivative of the objective function as a function of the policy operator is derived. By upper bounding the policy difference via the theory of integral probability metric, a new majorization-minimization type policy gradient algorithm for general MDPs is derived. This leads to generalization of many well-known algorithms in reinforcement learning to cases with general state and action spaces. Further, by taking the integral probability metric as maximum mean discrepancy, a low-complexity policy gradient algorithm is derived for finite MDPs. The new algorithm, called MM-RKHS, appears to be superior to PPO algorithm due to low computational complexity, low sample complexity, and faster convergence.  
\end{abstract}



\section{Introduction}
Proximal policy optimization has been one of the widely used algorithm for reinforcement learning with numerous successes in commercial and academic settings. Although the original derivation of PPO algorithm has been done for finite state-finite action MDPs, the algorithm has seen success even in continuous-state continuous-action settings.  This versatile algorithm has now been used to optimize performance metrics in disparate dynamical systems such as robotics, drug discovery, material discovery, ridesharing, active flow control, to name a few. The goal of this paper is to propose a policy-gradient type algorithm for computing an optimal policy for general-state general-action Markov decision processes (MDPs), where the state and the action spaces are complete separable metric spaces (Polish spaces). We refer to such MDPs as general MDPs. 

PPO works because it closely mimics majorization minimization algorithm, and therefore, enjoys the convergence benefit of the majorization minimization algorithm. In the majorization minimization algorithm for minimizing a function over the Euclidean space, one needs to obtain an upper bound (since we are considering minimizing cost) on the objective function in the vicinity of each iterate during the descent process. Typically, this is obtained by identifying an upper bound on the second order term in the Taylor series expansion of the objective function. As is well-known, this algorithm converges \cite[Proposition 7.3.1]{lange2016mm} under fairly general conditions when the optimization variable lies in Euclidean space.

The extension of PPO algorithm for general MDPs is not straightforward. The basis for PPO algorithm was TRPO algorithm \cite{schulman2015trust}, where the TRPO algorithm was derived for finite MDPs with a reward function that is only the function of state. That reference also was dependent on a series of work on policy gradient methods for finite MDPs, where the expressions are derived using derivative formula for functions over Euclidean spaces. To arrive at a majorization of the value function in the general MDP setting, we need to carefully exploit the results available in the infinite dimensional settings. Moreover, our goal is to arrive at an algorithm that can be used for complex applications involving general MDPs. Accordingly, in this paper, we put forth a new way of looking at general MDPs. We view the transition kernel and the policy as linear operators over carefully constructed function spaces. We first outline key engineering applications of general MDPs and then present key contributions.

\subsection{Motivating Applications}
In this subsection, we present the case for studying reinforcement learning algorithms for general MDPs. 

\paragraph{Active Flow Control}
Active flow control is an important technique in aerospace and automotive applications to enhance aerodynamic performance, reduce drag and improve controllability. Montala et al. \cite{montala2024active,suarez2025active} have studied the use of reinforcement learning to reduce the flow separation on aircraft wings at high angles of attack and high Reynolds number. The state space in this problem is the fluid flow field around the aircraft wing, which is an infinite dimensional space. 

\paragraph{Mean Field Control (MFC)}
In recent years, mean field games (MFGs)~\cite{HuangMalhameCaines_2006,lasry2007mean} have proven effective for modeling and coordinating large multi-agent systems arising in smart grids and finance. When these systems are fully cooperative (referred to as mean field control or mean field teams), the optimal solution requires solving a dynamic program characterized by a measure-valued state~\cite{carmona2013control,arabneydi2014team}. To navigate this complexity, recent research has begun applying reinforcement learning  algorithms to MFC settings~\cite{subramanian2018reinforcement,angiuli2022unified,anand2024mean}.


\paragraph{Soft Robotics Control Design}
Soft robots have been studied intensely in the past decade \cite{della2023model}. The configuration space of a soft robot is generally a function space. For example, in Cosserat rod model, the state is a map that maps the location on the rod (a unit interval) to three-dimensional configuration space (location of the that point on the rod in the space), and is a solution to a partial differential equation \cite{zheng2022pde}. Control of soft robotics remains a significant challenge and an active area of research \cite{della2023model}. Some recent work has also studied the problem of design of soft robotics using reinforcement learning \cite{baaij2023learning}.

Other relevant applications of reinforcement learning with general state and action spaces are in the fields of Thermal Field and Plasma Control \cite{zhou2025deep, degrave2022magnetic}, Reaction-Diffusion Pattern Formation \cite{schenk2024model,yoon2020design}, Vibration Control of Flexible Structures \cite{barjini2025deep}, etc. The control design for complex reward functions in these applications remains a significant challenge; if high fidelity simulators are available, we can use reinforcement learning based algorithm for controller design. This paper creates a theoretical approach to achieve this.

\subsection{Contributions}
In this paper, we pose the problem of Markov decision problems within the framework of linear operator theory, in which the state transition kernel and the policies are linear operators over appropriate weighted normed spaces. We demonstrate the existence of an optimal solution of the MDP using this theory, which differs significantly from other existing results in the literature. The key innovation in this paper is that we define two operators---one that maps the value function space to Q function space (state-action value function) and the operator that maps the Q function space to the value function space. Due to this new framework developed here, without employing the contraction mapping theorem, we are able to deduce the existence of optimal solution in linear-quadratic-Gaussian MDPs (LQG), finite MDPs, Lipschitz MDPs, and general MDPs within the same mathematical framework. We note here that this viewpoint is reminiscent of the approach adopted by Sch\"{a}l in \cite[Section 3]{schal1975dynamic} to devise the ws topology on the space of policies to demonstrate the existence of optimal policy in general MDPs. However, the author does not devise a computational approach in that paper. 

This viewpoint allows us to use the perturbation theory of linear operators \cite{kato2013perturbation,atkinson2005theoretical} over Banach spaces to derive the policy difference lemma for general MDPs. This result unifies the policy difference and policy gradient results across various prior works \cite{kakade2002approximately,schulman2015trust} and extends them to general MDPs \cite{lascu2025ppo}. This new policy difference lemma further allows us to exploit the theory of integral probability metrics to derive a majorization minimization algorithm for general MDPs. We note that integral probability metrics, particularly maximum mean discrepancy, is comparatively easier to compute even over infinite dimensional spaces \cite{han2023class}. We further demonstrate that for finite MDPs, the majorization minimization algorithm recovers the celebrated TRPO algorithm \cite{schulman2015trust}.

For the special case of finite MDPs with a slightly different majorization function involving mirror descent \cite{bertsekas2016nonlinear}, we arrive at a new class of reinforcement learning algorithms that are indexed by positive definite matrices. This new algorithm has an attractive computational property -- as soon as advantage function is estimated, one employs a simple computational algorithm using matrix algebra to derive the updated policy. This new policy satisfies the cost-improvement property, which is inherited due to the property of the majorization minimization algorithm. This obviates the need to solve a complex high dimensional optimization problem involving KL divergence, as is usually the case in TRPO algorithm. Thus, the algorithm designed here enjoys the benefits of TRPO algorithm without the computational overhead of computing the derivatives of the KL divergence.

\subsection{Notation}

Let $(\mathcal X,\mathfrak X)$ be a topological space in which $\mathcal X$ is the set and $\mathfrak X$ is the Borel sigma algebra on $\mathcal X$. We use $\mathcal P(\mathcal X)$ to denote the set of all probability measures over $\mathcal X$. Let $v:\mathcal X\to\Re$ be a measurable function. For $\rho\in\mathcal P(\mathcal X)$, define the pairing between a function and a probability measure as $\pair{v,\rho}=\int v(s) \rho(ds)$. 

\subsection{Organization}
The rest of the paper is organized as follows. In Sec.~\ref{sec:technical}, we present a brief overview of some of the technical tools used in the paper, including the theory of linear operators and integral probability metrics. The main optimization problem is formulated in Sec.~\ref{sec:problem}. In Sec.~\ref{sec:problem}, we provide sufficient conditions under which an optimal time-homogeneous policy exists. In Sec.~\ref{sec:policy-gradient}, we develop the policy difference lemma, a metric on policies, and majorization bounds. In Sec.~\ref{sec:policy-gradient}, these results are connected to policy iteration, PPO-style updates, trust region policy optimization, and majorization minimization methods. We conclude the discussions in Section \ref{sec:conclude}.

\section{Technical background}\label{sec:technical}

This section collects operator-theoretic background and the definition of integral probability metrics on weighted function spaces.
\subsection{Primer on Linear Operators}\label{sub:linearop}
Let $\mathcal X$ and $\mathcal Y$ be Banach spaces with norms given by $\|\cdot\|_{\mathcal X}$ and $\|\cdot\|_{\mathcal Y}$, respectively.

\paragraph{Operator Norm}

Let $\mathcal{B}(\mathcal X, \mathcal Y)$ denote the set of bounded linear operators $\op A:\mathcal X\to\mathcal Y$ equipped with the operator norm $\|\op A\| = \sup_{\|x\|_{\mathcal X}\leq 1} \|\op Ax\|_{\mathcal Y}$. The identity operator $\id{\mathcal X}\in \mathcal B(\mathcal X,\mathcal X)$ is given by $\id{\mathcal X}(x) = x$. If $\op A$ is invertible,  its inverse is denoted by $\op A^{-1}:\mathcal Y\to\mathcal X$.
For any subset $\mathcal X_0\subset\mathcal X$, its image under an operator $\op A \in \mathcal B(\mathcal X, \mathcal Y)$ is defined as 
$\op A\mathcal X_0 = \{\op Ax_0:x_0\in\mathcal X_0\}$. The following are well-known results. 

\paragraph{Compact Operator}

The linear operator $\op A$ is said to be a compact operator if it maps bounded subsets of $\mathcal X$ to precompact subsets in $\mathcal Y$. As an example, if $\mathcal Y$ is the Banach space of measurable functions with the sup norm, then due to Arzela Ascoli theorem, a precompact subset would be uniformly bounded and equicontinuous set of functions. An example of uniformly bounded equicontinuous set of functions is the set of H\"older continuous functions with finite upper bounds on the sup norm and H\"older coefficient. 

\paragraph{Spectral Radius:} For a linear operator $\op A\in\mathcal B(\mathcal X,\mathcal X)$, a complex number $\lambda\in\mathbb{C}$ belongs to the spectrum of $A$ if $\op A-\lambda \id{\mathcal X}$ is not invertible \cite[Def. 10.10, p. 234] {rudin1973functional}. The spectral radius of the operator, denoted by $\spec(\op A)$, is defined as the supremum of the magnitudes of all such values in the spectrum of the operator $\op A$. 

\paragraph{Neumann Series} For a linear operator $\op A\in\mathcal B(\mathcal X,\mathcal X)$, the Neumann series $\sum_{t=0}^\infty \op A^t$ converges to $(\id{\mathcal X}-\op A)^{-1}$ if $\|\op A\|<1$ \cite[Theorem 2.3.1]{atkinson2005theoretical}. More generally, the convergence holds if $\|\op A^m\|<1$ for some $m \in \Na$ \cite[Corollary 2.3.3]{atkinson2005theoretical}. The latter condition is satisfied when $\spec(\op A)<1$ by Theorem 10.13 on p. 235 of \cite{rudin1973functional}.

The following result is crucial for establishing the policy difference lemma in the paper. 
\begin{lemma}\label{lem:perturbop}
    Let $\op A,\op B\in\mathcal B(\mathcal X,\mathcal X)$ be two bounded linear operators such that $\op A^{-1}$ is also bounded. Then there exists $\bar\epsilon>0$ such that for every $\epsilon\in[-\bar\epsilon,\bar\epsilon]$, $\op  A+\epsilon \op B$ is an invertible operator with a bounded inverse and
    \begin{align}
        (\op A+\epsilon \op B)^{-1}  &= \op A^{-1}-\epsilon (\op A+\epsilon \op B)^{-1} \op B \op A^{-1} = \op A^{-1}-\epsilon  \op A^{-1}\op B (\op A+\epsilon \op B)^{-1} \label{eqn:diffinvop}\\
        & = \op A^{-1}-\epsilon \op A^{-1}\op B \op A^{-1} + \epsilon^2 \op A^{-1}\op B (\op A+\epsilon \op B)^{-1}  \op B \op A^{-1}\label{eqn:diffinvop2}
    \end{align}
\end{lemma}
\begin{proof}
    This result is well-known for matrices \cite{kato2013perturbation}, and is in fact derived and used in \cite{schulman2015trust}. For linear operators over Banach spaces, this result is stated in the proof of Theorem 2.3.5 in \cite[p. 67]{atkinson2005theoretical}. Equation \eqref{eqn:diffinvop2} is directly derived through an appropriate substitution in \eqref{eqn:diffinvop}.
\end{proof}

\paragraph{Inf-compact functions}    A measurable function $f:\mathcal X\to\Re$ is said to be inf-compact if the sublevel set of the function $\{x\in\mathcal X:f(x)\leq \ell\}$ is compact for every $\ell \in\Re$. Note that any inf-compact function is lower-semicontinuous since the sublevel sets are closed \cite{hernandez2012discrete}.

\subsection{Integral Probability Metrics} \label{sub:ipm}
We provide here an overview of integral probability metrics from \cite{zolotarev1984probability} and \cite{muller1997integral}. Let $\mathcal X$ be a complete separable metric space (Polish space) and $\mathcal M_{\mathcal X}$ be the set of measurable functions $f:\mathcal X\to\Re$. Given a weight function $\texttt{w} \in \mathcal M_{\mathcal X}$ with $\texttt{w}(x) \ge 1$ for all $x \in \mathcal{X}$, define the weighted norm for any $f \in \mathcal{M}_{\mathcal X}$ as $\|f\|_{\texttt{w}} = \sup_{x\in\mathcal X} |f(x)|/\texttt{w}(x)$. This induces a Banach space $\mathcal F_{\texttt{w}}:=\{ f \in \mathcal{M}_{\mathcal X} : \|f\|_{\texttt{w}} < \infty \}$ of measurable functions over $\mathcal X$ with bounded weighted sup norm. Define $\mathcal P_{\texttt{w}}(\mathcal X)$ as the set of probability measures $\mu$ over the space $\mathcal X$ such that $\int \texttt{w} d\mu <\infty$. Hence, for $f \in \mathcal{F}_{\texttt{w}}$ and $\mu \in \mathcal{P}_{\texttt{w}}(\mathcal{X})$, $\int f d\mu < \infty$. 

We define some properties of subset of functions.

\begin{definition}
    A set $\mathfrak{F}\subset\mathcal F_{\textttup{w}}$ is balanced if and only if $\{\alpha f: \alpha\in[-1,1], f\in\mathfrak{F}\}\subset \mathfrak{F}$.
\end{definition}

\begin{definition}
    A set $\mathfrak{F}\subset\mathcal F_{\textttup{w}}$ is absolutely convex if and only if it is convex and balanced.
\end{definition}

\begin{definition}
    A set $\mathfrak{F}\subset\mathcal F_{\textttup{w}}$ separates points in $\mathcal X$ if and only if for every $x_1,x_2\in\mathcal X$, there exists $f\in\mathfrak{F}$ such that $f(x_1)\neq f(x_2)$.
\end{definition}

Consider a function set $\mathfrak{F}\subset\mathcal F_{\textttup{w}}$ that separates points; this set $\mathfrak{F}$ is referred to as the generator set. The integral probability metric on $\mathcal{P}_{\texttt{w}}(\mathcal{X})$ (with respect to $\mathfrak{F}$) is defined as
\begin{align*}
    \ipm_{\mathfrak{F}} (\mu,\mu') = \sup_{f\in\mathfrak F} \Bigg| \int f d\mu - \int f d\mu'\Bigg| , \qquad \mu,\mu'\in\mathcal P_{\texttt{w}}(\mathcal X),
\end{align*}
which is a valid metric on $\mathcal{P}_{\texttt{w}}(\mathcal{X})$ \cite[Remark 4, p. 432]{muller1997integral}.

Define the absolutely convex operator $\absconv$ as 
\begin{align*}
    \absconv(\mathfrak{F}) = \bar{co}(\{\alpha f: f\in \mathfrak{F}, \alpha\in[-1,1]\}),
\end{align*}
where $\bar{co}(\cdot)$ is the closure of the convex hull of the set of functions $(\cdot)$. By construction, $\absconv(\mathfrak{F})$ is closed, absolutely convex and separates points. It is shown in \cite[Theorem 3.5]{muller1997integral} that $ \ipm_{\mathfrak{F}}  =  \ipm_{\absconv(\mathfrak{F})}$ and $\absconv(\mathfrak{F})$ is called the maximal generator. We note here that the closure here is in a specific topology, as discussed in \cite{muller1997integral}.

The Minkowski functionals $\varrho_{\mathfrak{F}}:\mathcal F_{\texttt{w}}\to[0,\infty]$ is defined as 
\[\varrho_{\mathfrak{F}}(f) = \inf\left\{r>0: \frac{f}{r}\in \absconv(\mathfrak{F})\right\},\]
where infimum of an empty set is taken to be $\infty$.

\section{Problem Formulation}\label{sec:problem}
We consider here a discounted-cost Markov decision process $(\mathcal S, \mathcal A, \mathcal K, P, c, \rho, \gamma)$, where
\begin{itemize}
    \item $\mathcal S$ and $\mathcal A$ are the state and action spaces, which are assumed to be complete separable metric spaces (Polish spaces), endowed with their respective Borel $\sigma$-algebras. 
    We endow the probability spaces $\mathcal P(\mathcal S)$ and $\mathcal P(\mathcal A)$ with the weak topology.  
    \item $\mathcal K \subset \mathcal S \times \mathcal A$ denotes the measurable set of feasible state-action pairs, i.e., for each $s \in \mathcal S$, $\mathcal K(s) \coloneqq \{ a\in \mathcal A : (s,a) \in \mathcal K \}\subseteq \mathcal A$ denotes the set of feasible actions at $s$. We assume that $\mathcal K$ and $\mathcal K(s)$, $s\in\mathcal S$, are Borel spaces.
    \item $P \colon \mathcal K\to \mathcal P(\mathcal S)$ denotes the transition kernel, which is assumed to be a Borel measurable map. 
    \item $c \colon \mathcal K \to \Re_{\ge 0}$ denotes the per-step cost function.
    \item $\rho \in \mathcal P(\mathcal S)$ denotes the initial state distribution.
    \item $\gamma \in [0, 1)$ is the discount factor.
\end{itemize}
At each time, the agent observes the state and chooses its action according to a time-homogeneous stationary policy $\pi:\mathcal S\to\mathcal P(\mathcal A)$, where $\pi(s)$ is a probability distribution supported over $\mathcal K(s)$ and is a Borel measurable map. We assume that $\pi$ is restricted to belong to a pre-specified subset  $\Pi \subset \{\pi:\mathcal S\to\mathcal P(\mathcal A):\pi(s) \subset\mathcal P(\mathcal K(s))\}$, which we call the set of structured policies. When its convenient, we let $\Pi$ be a structured class within the class of deterministic policies rather than randomized policies. Examples of structured policies could include linear policies, monotone policies, measurable policies etc.

The system evolves as follows. The initial state $s_0 \sim \rho$. Then, at time~$t$, the agent chooses action $a_t \sim \pi(s_t)$ and the state evolves as $s_{t+1} \sim P(\cdot | s_t, a_t)$ and the system incurs a cost $c(s_t, a_t)$. The performance of a policy $\pi \in \Pi$ is given by 
\[
   J_{\pi}(\rho) = \langle v_{\pi}, \rho \rangle 
   \text{ where }
    v_{\pi}(s_0) 
    = \exo{\sum_{t=0}^\infty \gamma^t c(s_t,a_t)\Big| s_0}{a_t\sim\pi(s_t)}.
\]

A policy $\pi^* \in \Pi$ is said to be optimal (with respect to the initial distribution $\rho$) if
\[
    J_{\pi^*}(\rho) \le J_{\pi}(\rho), \quad \forall \pi \in \Pi. 
\]
The value function corresponding to the optimal policy $\pi^*$ is called the optimal value function $v_{\pi^*}$. Generally speaking, if an MDP has multiple optimal policies in discounted setting, then all optimal policies lead to the same value function $v^*(s) = \inf_{\pi}J_{\pi}(\delta_{s})$ for all $s\in\mathcal S$; see Definition 9.3 and succeeding discussions in \cite{bertsekas1996stochastic}. Thus, we use $v^*$ to denote the unique optimal value function corresponding to any optimal policy. In the sequel, we identify a set of sufficient conditions under which optimal policies exist.

To illustrate the ideas of the paper, we will use the following running example. 
\begin{example}[Linear Quadratic Regulator (LQR)]\label{ex:LQR}
    Consider a system with linear dynamics $s_{t+1} = As_t + B a_t$, where $\mathcal S = \Re^n$, $\mathcal A = \Re^m$, $A\in\Re^{n\times n}$, $B\in\Re^{n\times m}$, and quadratic per-step cost $c(s,a) = s^\TRANS Q s + a^\TRANS R a$, where $Q \in \Re^{n \times n}$ is a symmetric positive semi-definite matrix, and $R \in \Re^{m \times m}$ is a symmetric positive definite matrix. There are no restrictions on the actions; thus $\mathcal K = \mathcal S \times \mathcal A$ and $\mathcal K(s) = \mathcal A$. 
    For the purpose of this paper, we assume that the structured policies are restricted to the set of linear policies, i.e., $\Pi = \{ \pi \colon \mathcal{S} \to \mathcal{A} : \pi(s) = -K s, K \in \Re^{m \times n} \}$. 
\end{example}
We refer the reader to~\cite{Caines2018} for a detailed overview of LQR. A standard result in LQR theory is that an optimal policy exists when the system is stabilizable. A system is said to be stabilizable if there exists a state feedback gain matrix $K$ for
which all the eigenvalues of $A - BK$ lie within the unit circle in the complex plane. 

\subsection{Preliminaries}

\subsubsection{Weighted norms, IPMs, and \texorpdfstring{$\lifs$}{} and \texorpdfstring{$\lifsa$}{} spaces}

Let $\ms$ and $\msa$ be the space of measurable functions on $\mathcal{S}$ and $\mathcal{K}$, respectively.  

Consider two measurable weight functions $\ws \in \ms$ and $\wsa \in \msa$ such that $\ws(s) \geq 1$ for all $s \in \mathcal{S}$ and $\wsa(s,a) \geq 1$ for all $(s,a) \in \mathcal K$. 
Let $\|\cdot\|_{\fs}$ and $\|\cdot\|_{\fsa}$ denote weighted norms on $\ms$ and $\msa$ with respect to $\ws$ and $\wsa$, i.e., 
for any $v \in \ms$ and $q \in \msa$, 
\begin{align*}
  \|v\|_{\fs} = \sup_{s\in\mathcal S} \frac{|v(s)|}{\ws(s)}, \qquad \|q\|_{\fsa} = \sup_{(s,a)\in\mathcal K} \frac{|q(s,a)|}{\wsa(s,a)}.
\end{align*}
Let $\fs$ and $\fsa$ denote subsets of $\ms$ and $\msa$ with finite weighted norm. Note that $(\fs, \|\cdot\|_{\fs})$ and $(\fsa, \|\cdot\|_{\fsa})$ are Banach spaces. We assume that $\fs\subset\fsa$ or $\|\ws\|_{\fsa}<\infty$.

Let $\mathcal{P}_{\ws}(\mathcal{S})$ and $\mathcal{P}_{\wsa}(\mathcal{K})$ denote the set of measures on $\mathcal{S}$ and $\mathcal{K}$ defined as in Sec.~\ref{sub:ipm}. Consider $\gen{\lifs} \subseteq \fs$ and $\gen{\lifsa} \subseteq \fsa$ which are set of functions that separates points. Let $\ipms$ and $\ipmsa$ denote the IPMs with respect to these generator sets and let $\varrho_{\gen\lifs}$ and $\varrho_{\gen\lifsa}$ denote the corresponding Minkowski functionals. Define $\lifs_{\textttup{b}}\subset \fs$ and $\lifsa_{\textttup{b}}\subset\fsa$ as the set of functions with finite Minkowski functionals:
\[
    \lifs_{\textttup{b}} = \left\{v\in\fs: \varrho_{\gen\lifs}(v)<\infty\right\},\qquad \lifsa_{\textttup{b}} = \left\{q\in\fsa: \varrho_{\gen\lifsa}(q)<\infty\right\}.
\]
We define $\lifs \subset \lifs_{\textttup{b}}$ and $\lifsa \subset \lifsa_{\textttup{b}}$ as structured spaces. The structured spaces would generally comprise non-negative functions or submodular non-negative functions etc. within the set of functions with finite Minkowski functionals.

\begin{examplecont}{ex:LQR}
    Consider weight functions $\ws(s) = 1+ s^\TRANS s$ and  $\wsa(s,a) = 1 + s^\TRANS s + a^\TRANS a$. In what follows, we use $Q\geq 0$ to denote that $Q$ is a positive semi-definite matrix of appropriate dimensions. We define the generator sets as
    \begin{align}
        \gen{\lifs} &= \{v\in\fs: v(s) = s^T Q s + b \leq \ws(s), Q\geq 0, b\in [-1,1] \} \\
        \gen{\lifsa} &= \{q\in\fsa: q(s,a) = s^T Q s + a^T R a + 2 s^T N a + b \leq \wsa(s,a), \\
        & \qquad \text{ where } Q\geq 0, R\geq 0, N\in\Re^{n\times m}, Q-NR^{-1}N^T\geq 0, b\in[-1,1] \} \notag .
    \end{align}
    It is clear from the expressions that $\gen{\lifs}$ and $\gen{\lifsa}$ separates points. We define $\lifs = \{v\in\fs: v(s) = s^T Q s, Q\geq 0\}$ as the space of quadratic functions in $s$ and 
    \[\lifsa = \{q\in\fsa: q(s,a) = s^T Q s + a^T R a + 2\lambda s^T A^T Q B a, Q, R\geq 0, \lambda \in [0,1]\}\]
   as the space of quadratic and jointly affine functions in $(s,a)$. It is well-known that if $Q\geq 0$ and $N = A^T Q B$, then $ Q-NR^{-1}N^T\geq 0$. It is further easy to see that the Minkowski functional of elements in $\lifs$ and $\lifsa$ are finite.
\end{examplecont}

\begin{remark}
    The spaces $\lifsb$ and $\lifsab$ are linear manifolds in $\fs$ and $\fsa$, respectively, and $(\lifsb,\|\cdot\|_{\fs})$ and $(\lifsab,\|\cdot\|_{\fsa})$ are normed vector spaces.
\end{remark}

If the spaces $\lifs$ and $\lifsa$, respectively, are closed subsets of the Banach spaces $(\fs,\|\cdot\|_{\fs})$ and $(\fsa,\|\cdot\|_{\fsa})$, then $(\lifs,\|\cdot\|_{\fs})$ and $(\lifsa,\|\cdot\|_{\fsa})$ are complete and are, therefore, Banach spaces. The proof follows from the definitions in \cite[p. 2, 130]{kato2013perturbation}. However, the construction of $\lifs$ and $\lifsa$ does not imply that they are closed subsets of $\fs$ and $\fsa$, respectively. Thus, they may not be complete, and therefore, are not Banach spaces is general.
    
We make the first assumption on the MDP here. 
\begin{assumption}[Well-posedness]\label{assm:wellposedness}
    The initial state distribution $\rho \in \mathcal{P}_{\ws}(\mathcal{S})$. The cost function $c \in \lifsa$ is non-negative and is inf-compact on $\mathcal{K}$.
\end{assumption}

\subsubsection{Operators}
We define the following linear operators corresponding to the dynamics and the policy and the Bellman operators:
\begin{itemize}
    \item The linear operator $\op P\colon\fs\to\fsa$ is defined as
    \[
        [\op P v](s,a) = \int_{\mathcal S} v(s') P(ds'|s,a).
    \]
    \item Given any policy $\pi \in \Pi$, define the linear operator $\op\pi\colon\fsa\to\fs$ as
    \[
        [\op\pi q](s) = \int_{\mathcal A} q(s,a) \pi(da|s)
    \]
    and define $\op P_{\pi}\colon \fs \to \fs$ as $\op P_\pi = \op \pi \op P$. Note that, $\op P_\pi$ may be equivalently expressed as
    \[
    [\op P_\pi v](s) = \int_{\mathcal S}v(s') P_\pi(ds'|s).
    \]
    where the stochastic kernel $P_{\pi} \colon \mathcal{P}_{\ws}(\mathcal S) \to \mathcal P_{\ws}(\mathcal S)$ is given by
    \[
        P_{\pi}(ds'|s) = \int_{\mathcal A} P(ds'|s,a) \pi(da|s).
    \]
    \item Given a policy $\pi \in \Pi$ and time~$t$, define the linear operator $\op P_{\pi}^t = (\op P_{\pi})^t$. Note that, $\op P_\pi^t$ may be equivalently expressed as
    \[
    [\op P_\pi^t v](s) = \int_{\mathcal S}v(s') P_\pi^t(ds'|s).
    \]
    \item Given a policy $\pi \in \Pi$, define the \emph{Bellman operator} $\op T_{\pi} \colon \fs \to \fs$  as
    \[
        \op T_\pi v = c_\pi + \gamma \op P_\pi v,
    \]
    where $c_\pi = \op \pi c\in\fs$. 
    \item Finally, define the \emph{optimality Bellman operator} $\op T \colon \fs \to \fs$ as
    \[
        [\op T v](s) = \inf_{p \in \mathcal P(\mathcal K(s))} \int \left(c(s,a)+\gamma \int v(s') P(ds'|s,a)\right) p(da).
    \]
    
    Note that unlike all the above operators, $\op T$ is non-linear. It is well-known that the optimal value function is a fixed point of the Bellman operator $\op T$ \cite[Section 9.4]{bertsekas1996stochastic}. 
\end{itemize}

There are two key challenges. First, $\lifs$ and $\lifsa$ are subsets of Banach space that are not necessarily complete. Second, we have not assumed that the operator $\op T$ is a contraction, as this typically requires strong assumptions on the transition kernel $P$ with respect to the weight function of a weighted norm~\cite{hernandez2012discrete}. Consequently, Banach contraction theorem cannot be directly applied to argue the existence of a fixed point of $\op T$.

\begin{remark}
    If the infimum exists in the Bellman operator $\op T$ for a given $v\in\fs$, then one has to pick a policy that is minimizing across all the states. In this case, a selector theorem is generally applied; e.g., measurable selection theorem \cite{feinberg2007optimality}, monotone selection theorem \cite{topkis1998supermodularity}, etc. In the above expression, the minimizing policy may not be in the structured policy class. However, we make the assumption in the sequel in Assumption \ref{assm:Pioptimal} that for any $v\in\lifs$, one can pick the policy generated to be in the structured class.
\end{remark}

The value function $v_{\pi}$ corresponding to the policy $\pi\in\Pi$ is written as
\begin{align}
    v_\pi &= \left[\sum_{t=0}^\infty \gamma^t \op P_\pi^t\right] c_\pi = \op T_\pi^\infty 0.
\end{align}
We note here that there is no reason to believe that for a given $\pi$, the limit in the expression above and the value function $v_\pi$ exists. This raises two questions.
\paragraph{Key questions}
We first answer the following two key questions:
\begin{enumerate}
    \item For what class of policies $\pi$ do we have $v_\pi\in\fs$?
    \item Can we demonstrate that the optimal value function satisfies $v^*\in\lifs$?
\end{enumerate}
We address the first question in the following subsection. To answer the latter question, we require certain assumptions and a more detailed analysis. This constitutes a major contribution of the paper. 

\subsection{Existence of Value Function for a Given Policy}
In this section, we define stable policies and demonstrate that the value function exist for stable policies. 

\begin{definition}[Spectrally Stable Policies]
    A policy $\pi \in \Pi$ is said to be spectrally stable if the spectral radius $\spec(\gamma\op P_\pi)<1$. Let 
    \begin{align}
        \specstable = \{\pi\in\Pi:\spec(\gamma\op P_\pi)<1\}
    \end{align}
    denote the set of spectrally stable policies. 
\end{definition}

\begin{examplecont}{ex:LQR}
If a linear system is stabilizable, then it implies that there is a matrix $K$ such that $\spec(A-BK)<1$. Thus, stabilizable linear system has nonempty $\specstable$.
\end{examplecont}

We now define another class of policies, which we call decaying cost policies. In this case, the accrued per stage discounted cost decays to 0 as time progresses. We make this precise below.
\begin{definition}[Decaying Cost Policies]
    A policy $\pi\in \Pi$ is said to be a decaying cost policy (with respect to the cost $c\in\lifsa$) if $\lim_{t\to\infty} \gamma^t (\op P \op\pi)^t c = 0$ and $\op P\op \pi$ is a compact operator. Let
    \begin{align}
        \decay = \left\{\pi\in\Pi:\lim_{t\to\infty} \gamma^t (\op P \op\pi)^t c = 0 \text{ and } \op P\op\pi \text{ is a compact operator} \right\}
    \end{align}
    denote the set of such policies.
\end{definition}

A sufficient condition for $\op P \op \pi$ to be a compact operator is as follows:

\begin{lemma}
    If $\op P$ is compact and $\|\op \pi\|<\infty$, then $\op P \op \pi:\fsa\to\fsa$ is compact and $\op P_\pi:\fs\to\fs$ is compact. Moreover, $\op \pi(\op P \op\pi)^t:\fsa\to\fs$ is also compact for every $t\in\Na$.
\end{lemma}
\begin{proof}
    This is established in Theorem 4.8 in \cite[p. 158]{kato2013perturbation}. The last asserton is established via the principle of mathematical induction.
\end{proof}

For $\pi\in\decay$ the spectral radius of $\gamma \op P_\pi$ can be greater than 1. However, its restriction to a certain subspace has spectral radius strictly less than 1.  Let $\overline{sub}\{\cdot\}$ denote closure of the subspace spanned by the vectors in the set $\{\cdot\}$ in the norm topology. Then, for a $\pi \in \decay$, the subspace 
\begin{align}\label{eqn:H}
    \mathcal H_\pi = \overline{sub}\{(\gamma^t \op \pi(\op P \op\pi)^t c)_{t\geq 0}\}\subset\fs
\end{align} 
may be viewed as the stable manifold of the operator $\gamma\op P\op\pi$. Let $\gamma \op P_\pi\big|_{\mathcal H_\pi}$ be the restriction of the operator within this subspace. We have the following result. 

\begin{lemma}\label{lem:restriction}
    For any $\pi\in\decay$, the spectral radius of the restricted operator $\gamma \op P_\pi\big|_{\mathcal H_\pi}$ is strictly less than one, i.e.,
    $\spec\big(\gamma \op P\op\pi\big|_{\mathcal H_\pi}\big)<1$.
\end{lemma}
\begin{proof}
    This is established in the proof of the theorem in \cite{suzuki1976convergence}.
\end{proof}

We note here that due to Hahn-Banach theorem, the operator $\gamma \op P_\pi\big|_{\mathcal H_\pi}$ can be extended to entire $\fs$; we do not study the construction here -- this approach could be useful in extending some of the results discussed later in the paper on policy gradient methods in general MDPs. We next define the set of stable policies as follows.
\begin{definition}[Stable Policies]
    The set of stable policies is defined as $\stable = \specstable \cup \decay$.
\end{definition}

We place the following assumption on the MDP to establish the existence of a value function of a policy.
\begin{assumption}\label{assm:existencevpi}
The set of stable policies $\stable$ is nonempty.
\end{assumption}

\begin{theorem}\label{thm:existencevpi}
 Suppose that Assumptions \ref{assm:wellposedness} and \ref{assm:existencevpi} are satisfied. If $\pi\in\stable$, then $v_\pi\in \fs$ exists and $\langle v_\pi, \rho\rangle \in\Re_{\geq 0}$. 
\end{theorem}
\begin{proof}
See Appendix \ref{app:existencevpi} for proof.
\end{proof}

The existence of the value function of a spectrally stable policy does not depend on the form of the cost function because the corresponding Neumann series is always well defined. See Lemma~\ref{lem:pi-stable} in Appendix~\ref{app:existencevpi}. On the other hand, the existence of the value function of a decaying cost policy relies on the compactness of the composite operator $\op P \op \pi$ and the form of the cost function. If $\op P \op \pi$ is not compact, then the value function $v_{\pi}$ may not exist (as is illustrated by the example in \cite{suzuki1976convergence}). When $\op P \op \pi$ is compact and $\gamma^t (\op P \op \pi)^t c$ decays to zero,  the spectral radius of a certain restricted operator is less than 1 in a certain subspace (Lemma~\ref{lem:restriction}), which leads to existence of the value function in that subspace. See Lemma~\ref{lem:pi-decay} in Appendix~\ref{app:existencevpi}.

\begin{examplecont}{ex:LQR}
    For LQR with linear policy $\pi(s) = -Ks$ for an appropriate matrix $K$, $\pi \in \specstable$ iff $\spec(A - BK) < 1$. On the other hand, $\pi \in \decay$ means that there is a policy leading to finite infinite-horizon discounted cost $\sum_{t=0}^\infty \gamma^t (A-BK)^t \Big[Q  +K^T R K\Big](A-BK)^t<\infty$, but the system may be internally unstable, that is, $\spec(A - BK) \geq 1$. This is the case, for example, if the unstable subspace of $(A-BK)$ is in the null-space of the $K$ and the $Q$ matrices. In this case, $\mathcal H_\pi$ is the intersection of the null space of $Q$ and the stable subspace of $(A-BK)$. Further, $\spec\Big((A-BK)|_{\mathcal H_\pi}\Big)<1$.
\end{examplecont}

\subsection{Main assumptions and their implications}

\subsubsection{Structured Space Assumptions}
We start by some basic assumptions on the model. 
In certain MDPs, there is a structure in the MDP allowing us to establish that the value functions, Q functions, and the policies lie is certain classes. 
Examples include submodular MDPs and monotone policies \cite{serfozo2009monotone,light2021stochastic}, $(s,S)$ policies in inventory management \cite{scarf1960optimality,feinberg2020stochastic}, LQG problems with linear policies \cite{athans1971role}, etc. to name a few. To this end, we make the following assumption.

\begin{assumption}[Compatibility of Operators and Linear Spaces]\label{assm:invariance}
  $\op P\lifs \subset\lifsa$ and $\op\pi \lifsa\subset\lifs$ for all $\pi\in\Pi$. 
\end{assumption}

As a consequence of Assumption \ref{assm:invariance}, $\lifs$ is an invariant manifold of the linear operator $\op P_\pi$ \cite[p. 132]{kato2013perturbation}. Assumption~\ref{assm:wellposedness} together with Assumption~\ref{assm:invariance} implies that $c_{\pi} \coloneqq \op\pi c \in \lifs$.

\begin{examplecont}{ex:LQR}
    Observe that for Example~\ref{ex:LQR}, for $v = s^T Qs$, we have $\op P v = (As+Ba)^T Q (As+Ba)$, which lies in $\lifsa$. For any policy $\pi \in \Pi$, we have $a = -Ks$. Substituting this in the above expression for $q\in\lifsa$, we get a quadratic expression in $s$. Thus, $\op \pi q \in \lifs$ and hence  Assumption~\ref{assm:invariance} is satisfied. 
\end{examplecont}

In the expression for Bellman operator $\op T$, we are minimizing a function in the space $\lifsa$. Generally speaking, if $c+\gamma \op P v$ is a lower semicontinuous function, the existence of infimum for all $s\in\mathcal S$ and the measurability of the $\op T v$ is studied under the umbrella of measurable selection theorem \cite{aliprantis2006infinite,feinberg2007optimality,bertsekas1996stochastic}. However, measurable selection theorem based results only imply that the $\op T v$ is a measurable map and lies in $\fs$. We are looking for a stronger result here that $\op T v\in\lifs$. Towards this end, we make the following assumption. 

\begin{assumption}[Structured selectors]\label{assm:Pioptimal}
    For every $q\in\lifsa$ such that $q$ is inf-compact, there exists a $\pi\in\Pi$ such that $\pi(s) \in \argmin_{p\in\mathcal P(\mathcal K(s))} \int q(s,a) p(da)$.
\end{assumption}

\begin{examplecont}{ex:LQR}
Pick $q(s,a) = s^T Q s + a^T R a + 2\lambda s^T A^T Q B a$. This is inf-compact if $R>0$, which further implies $R$ is invertible. By completing the square, we get
\begin{align*}
    q(s,a) = (a + \lambda R^{-1}  B^T Q A s)^T R (\lambda R^{-1}  B^T Q A s) + s^T (Q -  \lambda^2  A^T Q B R^{-1} B^T Q A) s.
\end{align*}
Thus, the minimizing $a^* = -\lambda R^{-1}  B^T Q A s$, which is precisely in the linear policy class. Thus, Assumption \ref{assm:Pioptimal} is satisfied by LQR. 
\end{examplecont}

We are now in a position to establish that the value iteration algorithm leads to uniformly bounded sequence of value functions.

\begin{theorem}\label{thm:uniformlybounded}
    If Assumptions \ref{assm:wellposedness}, \ref{assm:existencevpi},  \ref{assm:invariance} and  \ref{assm:Pioptimal} hold, then there exist $\kappa,\tilde\kappa>0$ such that for every $v\in\lifs$ satisfying $\|v\|_{\fs}\leq \tilde\kappa$, we have $\|\op T^k (v)\|_{\fs}\leq\kappa$ for all $k\in\Na$.
\end{theorem}
\begin{proof}
See Appendix \ref{app:uniformlybounded}.
\end{proof}

The typical approach to establishing the existence of an optimal policy rely on establishing that the value iteration converges to a fixed point of the Bellman operator $\op T$ \cite{bertsekas1996stochastic}. Most papers have used contraction mapping theorem to establish the convergence to fixed point. In contrast, we adopt the convergence proof in \cite{feinberg2007optimality}, where the authors established the fixed point result via limits of a sequence of functions that is increasing (that is, $v_0\leq v_1\leq \ldots$). In that paper, the value functions were assumed to be lower semicontinuous. If a limit exists in a sequence of increasing lower semicontinuous functions, then the limit is also lower semicontinuous. This is not generally true for all types of structured spaces. For example, a limit of increasing Lipschitz functions may not be Lipschitz (a simple example is $v_k(s) = -s^k$ when $\mathcal S = [0,1]$). Since the Bellman operator is defined on a structured space $\lifs$, we place the following assumption on $\lifs$ to establish that the fixed point of the Bellman operator exists and lies in $\lifs$. 

\begin{definition}[Space closed under bounded increasing limits]
    The space $\lifs$ is \textbf{closed under bounded increasing limits} iff for every increasing sequence $\{v_k\}_{k=0}^\infty \subset \lifs$ such that $v_0\leq v_1\leq \cdots $ with $\sup_{k\geq 0} \|v_k\|_{\fs} < \infty$, the pointwise limit $v_\infty(s) = \lim_{k\to\infty} v_k(s)$ is also an element of $\lifs$.
\end{definition}

\begin{assumption}[Closed under Bounded Increasing Limits]\label{assm:increasinglimits}
    $\lifs$ is closed under bounded increasing limits. 
\end{assumption}

\begin{example}
Let $\lifs = \{f\in\fs: Lip(f)\leq L\}$ be the space of Lipschitz functions, where $Lip(f)$ denotes the Lipschitz coefficient of the function $f$ and $L$ is a positive constant. This space is closed under bounded increasing limits. 
\end{example}
\begin{examplecont}{ex:LQR}
Let $\{v_k\}_{k\in\Na}\subset\lifs$ such that $v_1\leq v_2\leq \cdots$. There exists a sequence of positive semidefinite matrices $\{Q_k\}_{k\in\Na}$ such that $v_k(s) = s^T Q_k s$. Since $v_k\leq v_{k+1}$, we have $Q_{k+1}-Q_k$ is positive semidefinite. Let $\eta :=\sup_{k\geq 0} \|v_k\|_{\fs} < \infty$, which implies $\eta s^Ts \geq v_k(s)$ for every $k\in\Na$ and $s\in\mathcal S$. Then $\eta I - Q_k$ is positive semidefinite for all $k\in\Na$. Thus, $\{Q_k\}_{k\in\Na}$ is an increasing sequence of positive semidefinite matrices that are bounded from above and therefore, has a limit \cite[Theorem 1.1]{behrndt2010monotone}, \cite[Theorem 2.13]{hiai2014introduction}. Then, $\lifs$ is closed under bounded increasing limits.  
\end{examplecont}

\subsubsection{Continuity of Transition Kernel}

If $v\in\lifs$ is a measurable function, we need further assumptions on the transition kernel $P$ so that $\op P v$ is a continuous function. This will allow us to establish the inf-compactness of the function $c+\gamma \op Pv$ for every $v\in\lifs$. To this end, we define the continuity of the transition kernel. 

\begin{assumption}\label{assm:Pcontinuous}    
The transition kernel $P$ is continuous with respect to $\ipms$, that is  for every $(s,a)\in\mathcal K$, and for every sequence $(s_n,a_n)_{n\in\Na}\subset\mathcal K$ such that $(s_n,a_n)\to (s,a)$, $\lim_{n\to\infty} \ipms(P(\cdot|s,a),P(\cdot|s_n,a_n)) = 0$.
\end{assumption}

The continuity property of transition kernel has been assumed previously in \cite{hernandez2012discrete,feinberg2007optimality,feinberg2024average,kara2022near,kara2023convergence} under total variation norm, Wasserstein distance, and weak metric. The above definition of the continuity of transition kernel is essentially the same assumption, and it generalizes these seemingly different assumptions, as total variation norm, Wasserstein distance, and weak metrics are integral probability metric under certain assumptions on the state-action spaces and generators $\gen\lifs$. The following lemma demonstrates implication of this assumption.

\begin{lemma}\label{lem:Pcontinuous}
    If Assumption \ref{assm:Pcontinuous} holds and $v\in\lifs$, then $\op Pv:\mathcal K\to\Re$ is a continuous function. 
\end{lemma}
\begin{proof}
    Since $v\in\lifs$, $\varrho_{\gen{\lifs}}(v)<\infty$. Pick a sequence $(s_n,a_n)_{n\in\Na}\subset\mathcal K$ such that $(s_n,a_n)\to (s,a)\in\mathcal K$. We have the following inequality:
    \[\Bigg|\int v(s')P(ds'|s,a) - \int v(s') P(ds'|s_n,a_n)\Bigg| \leq \varrho_{\gen{\lifs}}(v) \ipms(P(\cdot|s,a),P(\cdot|s_n,a_n)).\]
    Using Assumption \ref{assm:Pcontinuous}, we conclude that $\op Pv(s_n,a_n) \to \op Pv(s,a)$, which concludes the proof.
\end{proof}

We now have all the essential ingredients to establish the existence of optimal solution for the MDP, and demonstrate that the value iteration algorithm converges to the optimal value function in the next subsection.

\section{Existence of Optimal Stationary Policy}\label{sec:existence}
In this section, we establish the existence of optimal stationary policy. This has been established under a variety of assumptions before in \cite{hernandez2012discrete,feinberg2007optimality,feinberg2012average,feinberg2024average} for discounted cost MDPs. One notable difference between earlier results and this result is that the previous results traded off very strong assumptions on the transition kernel with very weak assumptions on the cost function and vice-versa. The existence result in this paper sits in the middle of the two extremes. Since we are using structured classes of functions and policies, the theory of linear operators and integral probability metrics, the results can be readily applied to a wide range of settings by picking an appropriate class of policies, weight functions and generators of the integral probability metrics.

\subsection{The Value Iteration Algorithm and the Existence of an Optimal Policy}
In this section, we study the value iteration algorithm starting with 0 function, that is, $v_{k+1} = \op T v_k $ with $v_0 = 0$. To this end, we establish that the value iteration algorithm is well-defined--- that is, it yields a value function $v_k\in\lifs$---and converges to a limit. We further show that this limit is indeed the optimal value function $v^*$. This has been established earlier in Proposition 3.1 in \cite{feinberg2007optimality,feinberg2012average} under weakly continuous transition kernel assuming that the cost function is inf-compact. However, in that paper, it was shown that the value function $v_k$ is a lower semi-continuous function of the state. In the next theorem, we show that our assumptions imply that the value functions lie in $\lifs$, thereby differentiating the result from the result in \cite{feinberg2007optimality}.

\begin{theorem}\label{thm:valueiteration}
Suppose Assumptions \ref{assm:wellposedness}, \ref{assm:existencevpi}, \ref{assm:invariance}, \ref{assm:Pioptimal}, \ref{assm:increasinglimits}, and \ref{assm:Pcontinuous} hold. Consider the value iteration algorithm starting from $v_0 = 0$, where $v_{k+1} = \op T v_k$, $k \ge 0$. The generated value functions $v_k \in \lifs$ and converge to a limit (pointwise), and the limit is the optimal value function $v^*$. Furthermore, there exists an optimal policy $\pi^* \in \Pi$.
\end{theorem}
\begin{proof}
The proof proceeds along the same lines as the proof of Proposition 3.1 in \cite{feinberg2007optimality} and Theorem 2 in \cite{feinberg2012average}. We first prove that $v_k \in \lifs$ for all $k$ via induction.  Clearly $v_0\in\lifs$. This forms the basis of induction. Now assume that $v_k \in \lifs$ and define $q_k = c + \gamma \op P v_k$. 
\begin{itemize}
    \item Step 1: We have $\op P v_k\geq 0$ and is a continuous map over $\mathcal K$ due to Assumption \ref{assm:Pcontinuous} and Lemma \ref{lem:Pcontinuous}. Thus, $q_k = (c+\gamma\op P v_k) \in\lifsa$ is a lower semi-continuous map due to Assumptions \ref{assm:wellposedness} and \ref{assm:invariance}.
    \item Step 2: Due to Assumption \ref{assm:wellposedness}, $c$ is inf-compact and due to Step 1, $\gamma \op P v_k$ is continuous. Thus, $q_k = (c+\gamma \op P v_k)\in\lifsa$ is inf-compact since the set $\{(s,a)\in\mathcal K:[c+\gamma \op P v_k](s,a)\leq \lambda \}$ is a closed subset of the compact set $\{(s,a)\in\mathcal K:c(s,a)\leq \lambda \}$ for any $\lambda\in\Re$.
    \item Step 3: Due to Assumption \ref{assm:Pioptimal}, we conclude that there exists $\pi_k\in\Pi$ such that $\op Tv_k = \op T_{\pi_k}v_k$. 
    Thus, $v_{k+1} = \op T_{\pi_k} v_k = \op \pi_k q_k$ belongs to $\lifs$ due to Assumption~\ref{assm:invariance}.
\end{itemize}
This completes the induction step. Hence $v_k \in \lifs$ for all $k \in \Na$. 

\noindent We now establish that $\{v_k\}_{k \ge 0}$ converges to a limit. 
\begin{itemize}
    \item  Observe that $\op Tv_k\geq v_k$ since the cost is non-negative \cite[p. 230]{bertsekas1996stochastic}. Thus, $\{v_k\}_{k\geq 0}$ is an increasing sequence of functions.
    \item Due to Assumptions \ref{assm:wellposedness}, \ref{assm:existencevpi}, \ref{assm:invariance} and  \ref{assm:Pioptimal} and Theorem \ref{thm:uniformlybounded}, we have  $\sup_{k\geq 0}\|v_k\|_{\fs} <\infty$. 
    \item Due to Assumption \ref{assm:increasinglimits}, we conclude that the limit of the iteration $v_{k+1} = \op T v_k$ converges to  a limit in $\lifs$.
\end{itemize}
Let $v^\circ\in\lifs$ denote the limit point of $\{v_k\}_{k\in\Na}$. Due to Proposition 9.17 in \cite[p. 234]{bertsekas1996stochastic}, we conclude that $v^\circ = \op T v^\circ$. Thus, $v^\circ$ is a fixed point of $\op T$. Since $v_0 =0$, $c$ is non-negative, and $q_k = c+\gamma\op P v_k$ is an inf-compact function (see Step 2 above), Proposition 9.17 in \cite[p. 234]{bertsekas1996stochastic} implies that the limit point $v^\circ$ is the optimal value function $v^*$. 

Finally, due to Assumption~\ref{assm:Pioptimal}, there exists a $\pi^* \in \Pi$ such that $v^* = \op T v^* = \op T_{\pi^*} v^*$. By, Proposition 9.12 in \cite[p. 227]{bertsekas1996stochastic}, such a policy $\pi^*$ is optimal. The proof is hence complete.
\end{proof}

\begin{remark}\label{rem:boundedcost}
    While Theorem~\ref{thm:valueiteration} is derived under the assumption that $c \ge 0$, the result extends to any cost function bounded from below; we can always add a constant to the cost to make it non-negative without changing the optimal policy.
\end{remark}

\subsection{Reward Shaping Using Potential Functions}
Reward shaping is an important technique in MDPs to arrive at an approximately optimal solution in fewer iterations. This was first studied in \cite{ng1999policy}. It was shown that potential-based reward function leads to a class of MDPs with the same optimal solution in finite-state finite-action MDP. We demonstrate that the same idea works for the general MDP as well, as long as the potential function satisfies certain criteria. The key result is presented below.

\begin{theorem}\label{thm:shaping}
    Suppose that $\Phi\in\lifs$ satisfies $\tilde c := c+\gamma\op P \Phi - \Phi\in\lifsa$ and $\lim_{t\to\infty}\gamma^t\op P_\pi^t\Phi = 0$ for all $\pi\in\decay$. If the hypotheses for Theorem \ref{thm:valueiteration} are satisfied and $\tilde c$ is bounded from below and inf-compact, then the optimal policy of the MDP with cost function $\tilde c$ is the same as the optimal policy for the MDP with the cost function $c$.
\end{theorem}
\begin{proof}
    Let $\tilde{\op T}_\pi$ and $\tilde{\op T}$ denote the Bellman operators with the cost function as $\tilde c$. The fact that value iteration converges for the MDP with the cost function $\tilde c$ follows from showing that the hypotheses of Theorem \ref{thm:valueiteration} hold for the MDP with the cost function $\tilde c$ as well (see Remark \ref{rem:boundedcost} above). Moreover, one can show that $\tilde v_\pi = v_\pi - \Phi$ is a fixed point of the Bellman operator $\tilde{\op T}_\pi$ for any $\pi\in\stable$. By the same argument as in  the proof of Theorem \ref{thm:valueiteration}, the limit $\lim_{k\to\infty} \tilde{\op T}^k 0$ exists and it is the optimal value function $\tilde v^* := \lim_{k\to\infty} \tilde{\op T}^k 0$. Since $\Phi$ is only a function of state, one can show that the optimal policy corresponding to $\tilde v^*$ in the MDP with cost $\tilde c$ is the same as the optimal policy corresponding to $v^*$ in the MDP with the cost function $c$ (in fact, $\tilde v^* = v^* = \Phi$). This concludes the proof. 
\end{proof}

\begin{remark}
    In the theorem above, if $\Phi\in\lifs$ is a non-negative upper semicontinuous function and $c-\Phi$ is bounded from below, then $\tilde c$ is inf-compact and bounded from below. Further,
    any $\Phi$ that satisfies $\Phi\in\cap_{\pi\in\decay} \mathcal H_\pi$ implies  $\lim_{t\to\infty}\gamma^t\op P_\pi^t\Phi = 0$ for all $\pi\in\decay$.
\end{remark}
Note that we did not assume boundedness of the cost function in our general MDP setting. To the best of our knowledge, this is perhaps the first result on the use of potential function for reward shaping in general MDPs with unbounded costs. The key assumption here is that the potential function must come from the structured class and preserve the structure in the new cost function.

\begin{remark}
    Suppose that $\hat v_k$ is the estimate of the value function at iteration $k$. It has been shown previously in \cite{bertsekas2012lambda} that if we pick $\Phi = \lambda \hat v_k$, then the resulting algorithm is equivalent to $\lambda$ policy iteration (also known as $\textup{TD}(\lambda)$).
\end{remark}

\section{Policy Difference Lemma and Derivatives of the Value Function}\label{sec:policy-difference}
In the sequel, we assume that the hypotheses for Theorem \ref{thm:valueiteration} hold. This ensures that an optimal policy exists and the value iteration algorithm converges to the optimal value function. In addition, we assume that $\specstable$ is nonempty because some of the results do not apply to the decaying policy class $\decay$.  
To derive various policy gradient based algorithm to compute the optimal policy, we need to derive an expression for the Gauteaux derivative of the value function as a function of the policy. To derive the Gateaux derivative, we need to establish the policy difference lemma for the general MDP. This result is not new; for finite MDPs, this has been established in \cite{sutton1999policy,schulman2015trust} under certain restrictions and for continuous MDPs, a version of this result has appeared in \cite{ju2022policy} (see Lemma 1). In the latter paper, the state and action spaces are assumed to be subsets of appropriate Euclidean spaces and the derivation does not appeal to the linear operator framework adopted here. 

\begin{definition}
    The q function at a policy $\pi$ is denoted by $q_\pi = c+\gamma \op P v_\pi$. The advantage function at a policy $\pi$ is defined as $A_\pi = q_\pi - v_\pi$.
\end{definition}

\begin{lemma}
    For any $\pi,\pi'\in\specstable$, $\op \pi' A_\pi\in\lifsb$. 
\end{lemma}
\begin{proof}
 Since $\op \pi' q_\pi\in\lifs$ due to Assumption \ref{assm:invariance}, the proof then follows from the definitions.
\end{proof}

\begin{definition}[Policy Advantage Function]
    The policy advantage function $L_\pi:\specstable\to\lifsb$ at a policy $\pi\in\specstable$ is defined as
\begin{align}\label{eqn:lpipi'}
    L_{\pi}(\pi') & = \exo{\sum_{t=0}^\infty \gamma^t[\op \pi' A_\pi](s_t)}{s_t\sim P_\pi(\cdot|s_{t-1})}  = \sum_{t = 0}^\infty \gamma^t \op P_\pi^t \op \pi' A_\pi =\op \sigma_\pi \op \pi' A_\pi.
\end{align}
\end{definition}

The policy difference lemma is proved below using the linear operator framework developed here. This result was established in the context of finite MDPs in Theorem 4.1 in \cite{kakade2002approximately} and for general MDPs in \cite{ju2022policy,lascu2025ppo}. Since we are using structured spaces, we need to establish this result in this context; however, since we are using perturbation theory for linear operators here, the proof is compact in comparison to the ones in \cite{kakade2002approximately,ju2022policy,lascu2025ppo}. 

\begin{lemma}[Policy Difference Lemma \cite{kakade2002approximately,lascu2025ppo,ju2022policy}]\label{lem:policydiff}
Let $\pi,\pi'\in\specstable$ and $\op \Delta = \op \pi' - \op \pi$ be the difference of the two policy operators. Then, the difference of the two corresponding value functions is
\begin{align*}
    v_{\pi'} - v_{\pi} = \op\sigma_{\pi'}\op \Delta q_\pi & =   L_{\pi}(\pi')  + \gamma \op\sigma_\pi \op P_\Delta \op \sigma_{\pi'} \op\Delta q_\pi \\
    & =   L_{\pi}(\pi')  + \gamma \op\sigma_\pi \op\Delta\op P \op \sigma_{\pi'} \op\pi' A_\pi\\
    & = L_{\pi}(\pi') + \gamma \op\sigma_\pi \op P_\Delta \op \sigma_{\pi} \op\Delta q_\pi + o(\|\op\Delta\|^2).
\end{align*}
\end{lemma}
\begin{proof}
See Appendix \ref{app:policydiff}.
\end{proof}

Let $\pi,\pi'\in\specstable$. If the directional derivative of the value function $v_\pi$ in the direction of $\pi'$ is linear and continuous, then it is called the Gateaux derivative $\gateaux_{\pi'}(v_{\pi})$ of the value function \cite[p. 34]{bonnans2000perturbation}. This is defined as
\begin{align*}
    \gateaux_{\pi'}(v_{\pi}) = \lim_{\epsilon\downarrow 0} \frac{1}{\epsilon}(v_{\pi+\epsilon(\pi'-\pi)} - v_{\pi}). 
\end{align*}
The policy difference lemma established above allows us to derive the directional derivative, and establish that it is linear and continuous in $\pi'$, thereby establishing that it is Gateaux derivative of the value function. 
\begin{lemma}
    If $\pi,\pi'\in\specstable$, then $\gateaux_{\pi'}(v_{\pi}) = L_{\pi}(\pi')$.
\end{lemma}
Thus, Gateaux derivative of the value function evaluated at a stable policy is precisely the policy advantage function. To the best of our knowledge, this is the first time this result has been established for general MDPs. A necessary condition for optimality of the value function is that the Gateaux derivative is $\gateaux_{\pi'}(v_{\pi^*}) = 0$ for all $\pi'\in\specstable$ (note that the 0 here refers to the zero function in $\fs$); this is established in Theorem 1 in \cite[p. 178]{luenberger1969optimization}.  

\subsection{A Metric structure on \texorpdfstring{$\Pi$}{}}\label{app:metricPi}
The operator norm on the space of policies $\Pi$ is $\|\op \pi\| = \sup_{\|q\|_{\fsa}\leq 1} \|\op\pi q\|_{\fs}$. This operator norm defines a metric on $\Pi$ when we restrict $q\in\lifsa$. To see this, for every $s\in\mathcal S$, we first define the generator set $\gen{\mathcal A}(s)$ and then define an integral probability metric on $\mathcal P_{\gen{\mathcal A}(s)}(\mathcal K(s))$. For every $s\in\mathcal S$, define the set of functions $\gen{\mathcal A}(s)$ as
\begin{align*}
    \gen{\mathcal A}(s) = \{h:\mathcal K(s)\to\Re: h(a) = q(s,a), q\in\gen\lifsa\}.
\end{align*}
We have the following result.
\begin{lemma}
    If $\gen\lifsa$ separates points, then $\gen{\mathcal A}(s)$ also separates points. If $\gen\lifsa$ is absolutely convex, then $\gen{\mathcal A}(s)$ is also absolutely convex.
\end{lemma}
\begin{proof}
The proof is obvious.
\end{proof}
The integral probability metric on $\mathcal P_{\gen{\mathcal A}(s)}(\mathcal K(s))$ is now defined as
\begin{align}
    \ipmas\ (\mu,\mu') = \sup_{h\in \gen{\mathcal A}(s)} \Bigg| \int h d\mu - \int h d\mu'\Bigg| , \qquad \mu,\mu'\in\mathcal P_{\gen{\mathcal A}(s)}(\mathcal K(s))
\end{align}
Let us now define the following weighted integral probability metric on the space of policies $\Pi$:
\begin{align}\label{eqn:metricPi}
    \metric(\pi,\pi') = \sup_{s\in\mathcal S} \frac{\ipmas\ (\pi(\cdot|s), \pi'(\cdot|s))}{\ws(s)}
\end{align}
We have the following result.
\begin{lemma}\label{lem:metricPi}
    For any $q\in \lifsa$, we have $ \|\op\pi q - \op\pi'q\|_{\fs} \leq \varrho_{\gen\lifsa}(q)\: \metric(\pi,\pi')$.
\end{lemma}
\begin{proof}
    For any $q\in \lifsa$, $q\neq 0$ and $\varrho_{\gen\lifsa}(q)\neq 0$, we have $\frac{q}{\varrho_{\gen\lifsa}(q)} \in \absconv(\gen{\lifsa})$. Thus, 
    \[\|\op\pi q - \op\pi'q\|_{\fs} = \varrho_{\gen\lifsa}(q)\Bigg\|\op\pi \frac{q}{\varrho_{\gen\lifsa}(q)} - \op\pi'\frac{q}{\varrho_{\gen\lifsa}(q)}\Bigg\|_{\fs}.\]
    Moreover, by definition of $\gen{\mathcal A}(s)$, we have $q(s,\cdot)/\varrho_{\gen\lifsa}(q)\in\absconv(\gen{\mathcal A}(s))$ for every $s\in\mathcal S$. This readily leads us to the desired inequality using the definition of the metric on $\Pi$ in \eqref{eqn:metricPi} and \cite[Theorem 3.5]{muller1997integral}.
\end{proof}
The preceding result allows us to derive a majorization bound on the right side expression in the policy difference lemma. This is achieved in the next subsection.

\subsection{Majorization Bounds on the Second Term}
Majorization of the value function $v_\pi$, as a function of policy $\pi$, is a function of policy $\pi$ that provides a global upper bound on the function $v_\pi$ and equals the function $v_{\tilde \pi}$ at a specific policy $\tilde \pi$. The goal for this section is to derive a majorization bound using the policy difference lemma established above. This will help us derive the majorization minimization algorithm in the sequel. We make the following assumption.

\begin{assumption}\label{assm:rhoPv}
     There exists a universal constant $\kp>0$, dependent only on $\op P$, such that for every $v\in\lifs$, we have $\varrho_{\gen{\lifsa}}(\op Pv)\leq \kp \|v\|_{\fs}$.
\end{assumption}

\begin{examplecont}{ex:LQR}
 Let $v(s) = s^T Q s$ and we get $\|v\|_{\fs} = \|Q\|_2$, where $\|\cdot\|_2$ is the 2-norm of the matrix (maximum singular value).
Using Young's inequality (see Section 2.8.1 in \cite{caverly2019lmi}) and submultiplicativity of norm of matrix multiplication \cite[p. 341]{horn2012matrix}, we get
\begin{align*}
(As+Ba)^T Q (As+Ba) &= a^T (B^T Q B) a + 2 a^T (B^T Q A) s + s^T (A^T Q A) s\\
&\leq 2 \left( s^T A^T Q A s + a^T B^T Q B a \right).\\
\implies \|(As+Ba)^T Q (As+Ba)\|_{\fsa} &\leq 2( \|A\|_2^2 + \|B\|_2^2) \|Q\|_2
\end{align*}
Thus, $\kp =2( \|A\|_2^2 + \|B\|_2^2)$.
\end{examplecont}

\begin{lemma}\label{lem:upperbound}
    Let  $\op F \in \mathcal B(\fs,\fs)$ be a bounded linear operator satisfying $\op F \lifs\subset\lifs$. Pick $\pi,\pi'\in\specstable$ and $q\in\lifsa$. Let $\op\Delta = \op\pi' - \op\pi$. The following holds:
    \begin{enumerate}
        \item $\|\op \sigma_{\pi}\|\leq (1-\|\gamma\op P_\pi\|)^{-1}<\infty$ for all $\pi\in\specstable$.
        \item $\|\op\Delta q\|_{\fs} \leq \metric(\pi,\pi')\varrho_{\gen\lifsa}(q)$.
        \item If Assumption \ref{assm:rhoPv} also holds, then $\|\op\Delta \op P \op F \op\Delta q\|_{\fs} \leq \|\op F\| \kp \varrho_{\gen\lifsa}(q) \metric(\pi,\pi')^2$.
        \item Let $\beta(q,\pi') = \kp \|\op \sigma_{\pi'}\| \varrho_{\gen\lifsa}(q)$. Then, the following inequality holds: 
        \begin{align}
        [\op\Delta \op P \op \sigma_{\pi'} \op\Delta q_\pi](s)
        &\leq \beta(q,\pi') \ws(s) \metric(\pi,\pi')^2 \text{ for all } s\in\mathcal S.\label{eqn:trpo}
\end{align}
    \end{enumerate}
\end{lemma}
\begin{proof}
    The proof of the first assertion follows from Corollary 2.3.3 in \cite{atkinson2005theoretical}. The second assertion is established Lemma \ref{lem:metricPi}. The third assertion follows from the following inequalities:
    \begin{align*}
        \|\op\Delta \op P \op F \op\Delta q\|_{\fs} &\leq \metric(\pi,\pi')\varrho_{\gen\lifsa}(\op P \op F \op\Delta q) \tag{Lemma \ref{lem:metricPi}}\\
        & \leq \metric(\pi,\pi')\kp \|\op F \op\Delta q\|_{\fs}\tag{Assumption \ref{assm:rhoPv}}\\
        & \leq \metric(\pi,\pi')\kp \|\op F\|\|\op\Delta q\|_{\fsa}\tag{Norm of Linear Operators}\\
        & \leq \|\op F\| \kp \varrho_{\gen\lifsa}(q) \metric(\pi,\pi')^2. \tag{Lemma \ref{lem:metricPi}}
    \end{align*} 
    The fourth assertion is a direct consequence of the first and the third assertion. In the third inequality above, one can use Lemma \ref{lem:Deltaq} to conclude \eqref{eqn:trpo}.
\end{proof}

The last inequality above allows us to derive the following majorization bound. 

\begin{theorem}\label{thm:majorization}
Let $\pi,\pi'\in\specstable$. Then, the majorization bounds on the difference of the value functions and the policy performance are given by
\begin{align*}
    [v_{\pi'} - v_{\pi}](s) &\leq \Big[L_{\pi}(\pi') + \op \sigma_\pi\beta(q_\pi,\pi') \ws \metric(\pi,\pi')^2\Big](s)\quad \text{ for all } s\in\mathcal S,\\
    J_{\pi'}(\rho) - J_{\pi}(\rho) &= \pair{v_{\pi'} - v_{\pi},\rho}\leq \pair{\op \pi'A_\pi + \beta(q_\pi,\pi') \ws \metric(\pi,\pi')^2,\op \sigma_\pi^*\rho},
\end{align*}
where $\op\sigma_\pi^*$ is the adjoint of the operator $\op\sigma_\pi$.
\end{theorem}
\begin{proof}
The proof follows from the preceding discussions in Lemmas \ref{lem:policydiff} and \ref{lem:upperbound}.
\end{proof}

\section{Policy Gradient Algorithms}\label{sec:policy-gradient}
In this section, we present some of the main algorithms used in reinforcement learning and connect those algorithms to the first and second order methods for policy optimization. The key results enabling these connections are derived in the previous section. We first need an interchange of integral and minimum assumption, so that we can conclude 
\begin{align*}
    \pi_{k+1} &\in \argmin_{\pi'\in\Pi}  \pair{\op \pi'q_k,\op\sigma_{\pi_k}^*\rho} \iff \pi_{k+1}(s) \in \min_{p\in\mathcal P(\mathcal K(s))} \int q_k(s,a) p(da)
\end{align*}
for almost every $s\in \supp(\op\sigma_{\pi_k}^*\rho)$, where $\op\sigma_{\pi_k}^*$ is the adjoint of the operator $\op\sigma_{\pi_k}$. We shed some light on this topic in the next subsection.

\begin{remark}
In the context of sampling based reinforcement learning, we can approximate the pairing $\pair{\op \pi'q_k,\op\sigma_{\pi_k}^*\rho}$ by
\begin{align*}
    \pair{\op \pi'q_k,\op\sigma_{\pi_k}^*\rho} \approx \sum_{i=0}^N \int q_k(s_i,a) \pi'(da|s_i),
\end{align*}
where $\{s_i\}_{i=0}^N$ is generated by $s_0\sim\rho$ and $s_{i+1}\sim P_{\pi_k}(\cdot|s_i)$, and $N$ is distributed according to geometric distribution with parameter $\gamma$, that is, $\pr{N=n} = (1-\gamma)\gamma^n$. It is a routine calculation to show that the right side is an unbiased estimate of the left side \cite{zhang2020global}.
\end{remark}

\begin{remark}
Let $F:\specstable\to\lifs$ be a function. It is easy to see that if the interchange of integral and minimization is valid and that the minimum exists,
    \[\argmin_{\pi'\in\Pi}  \pair{\op \pi'q_\pi + F(\pi'), \op\sigma_{\pi_k}^*\rho} = \argmin_{\pi'\in\Pi}  \pair{\op \pi'A_\pi + F(\pi'),\op\sigma_{\pi_k}^*\rho}.\]
In the sequel, we use this idea at several occasions without mention.
\end{remark}
\subsection{Interchange of Integral and Minimization}
Recall that we are interested in minimizing $\langle \op\pi' q,\xi\rangle$ for some $\xi\in\mathcal P_{\ws}(\mathcal S)$, which is minimization of an integral over a subset of linear operators (generally, $\specstable$). Under certain conditions, this minimization operation can be ``pushed'' inside the integral, whereby we can minimize the function $q$ over the actions $a$ and construct a $\pi'$ that minimizes $\langle \op\pi' q,\xi\rangle$. We first introduce a definition of interchange of integral and minimization and present a condition under which this is a feasible operation.

\begin{definition}
We say that $\lifsa$ and $\Pi_0$ enjoy the property of interchange of integral and minimization if and only if for any $q\in\lifsa$ that is inf-compact and for every  $\xi\in\mathcal{P}_{\ws}(\mathcal{S})$, the following holds:
    \begin{align*}
        \pi^*_{q,\xi} \in \argmin_{\pi\in\Pi_0} \int q(s,\pi(s)) \xi(ds) \iff \pi^*_{q,\xi}(s) \in \argmin_{a\in\mathcal K(s)} q(s,a) \text{ for $\xi$ a.e. } s\in\mathcal S.
    \end{align*}
\end{definition}
Early results on this property appeared in \cite{rockafellar2006integral}, which was further extended in \cite{hiai1977integrals,bouchitte1988integral}. Further extensions under more generality is studied in \cite{hafsa2003interchange} and \cite{giner2009necessary}. We make the following assumption in the paper. 

\begin{assumption}\label{assm:interchange}
    $\lifsa$ and $\Pi$ enjoy the interchange of integral and minimization property. 
\end{assumption}
When the action space is a Euclidean space and there is no structural restriction on the policy class, this assumption is readily satisfied as established in Lemma \ref{lem:interchange} below. 
\begin{lemma}\label{lem:interchange}
    Suppose that $\mathcal A = \Re^n$. Let $\Pi$ denote the set of all possible measurable policies $\pi:\mathcal S\to \mathcal A$.  Then, $\lifsa$ and $\Pi$ enjoy the property of interchange of integral and minimization.
\end{lemma}
\begin{proof}
    If $q$ is inf-compact, then the epigraph is a closed set; thus, $q$ is a normal integrand (see Definition 14.27 in \cite[p. 661]{rockafellar1998variational}). The proof then follows from Theorem 14.60 in \cite[p. 677]{rockafellar1998variational}. The second part of the statement follows immediately from Assumption \ref{assm:Pioptimal}.
\end{proof}

If Assumption \ref{assm:Pioptimal} holds in case $\Pi$ is a structured policy class, then $\pi^*_{q,\xi}\in\Pi$. Thus, if Assumption \ref{assm:Pioptimal} holds, then $\lifsa$ and $\Pi$ enjoy the property of interchange of integral and minimization. Theorem 2.2 in \cite{hiai1977integrals} establishes a similar result as in Lemma \ref{lem:interchange} for more general case when the action spaces are Polish and $\mathcal K$ satisfies certain conditions. 

\subsection{First order Policy Gradient Based Methods}

\paragraph{Policy iteration} Minimizing $\exo{L_{\pi_k}(\pi')(s)}{s\sim\rho} = \pair{\op\pi' A_{\pi_k}, \op\sigma_{\pi_k}^*\rho}$, that is the expected policy advantage function, leads to policy iteration algorithm; thus, policy iteration is equivalent to the usual gradient descent in the policy space, as can be observed below:
\begin{align*}
    \pi_{k+1} &= \argmin_{\pi'\in\Pi}  \pair{\op \pi'A_{\pi_k} ,\op\sigma_{\pi_k}^*\rho} \\
    \iff
    \pi_{k+1}(s) &= \argmin_{p\in \mathcal P(\mathcal K(s))} \int A_{\pi_k}(s,a) p(da) = \delta_{\mu_{k+1}(s)}, \text{ where } \mu_{k+1}(s) = \argmin_{a\in \mathcal K(s)} q_{\pi_k}(s,a)
\end{align*}
The objective vanishes at $\pi_k$; thus, $\pi_{k+1}$ necessarily does not degrade performance. This leads to a sequence of improving policies and if a policy is not improved at the next step, that is, $\pi_{k+1} = \pi_k$, then it is an optimal policy \cite{puterman2014markov}. 

\paragraph{Proximal policy optimization} PPO is a conservative policy optimization algorithm, where we minimize an upper bound on the advantage function estimate. The descent direction is
\begin{subequations}\label{eq:PPO}
\begin{align}
    \pi_{k+1} &= \argmin_{\pi'\in\Pi}  \pair{\op \pi_k\:\texttt{CPI}\left(\frac{d\pi'}{d\pi_k} , A_{\pi_k}\right) ,\op\sigma_{\pi_k}^*\rho},\\
    \text{where } \texttt{CPI}\left(\frac{d\pi'}{d\pi_k} , A_{\pi_k}\right) & = \max\left\{\frac{d\pi'}{d\pi_k} A_{\pi_k}, \texttt{clip}\left(\frac{d\pi'}{d\pi_k} ,1-\epsilon,1-\epsilon\right) A_{\pi_k}\right\},\\
    \texttt{clip}(x,1-\epsilon,1+\epsilon) &= \begin{cases}
        x & x\in[1-\epsilon,1+\epsilon]\\
        1-\epsilon & x<1-\epsilon\\
        1+\epsilon & x>1+\epsilon
    \end{cases},
\end{align}
\end{subequations}
$\frac{d\pi'}{d\pi_k}(s,a)$ is the Radon-Nikodym derivative between $\pi'(da|s)$ and $\pi_k(da|s)$ evaluated at action $a$, and $\epsilon$ is a hyperparameter. The operation $\texttt{CPI}$ provides a conservative estimate of the advantage function at low cost actions at certain states. It is clear, however, that the objective function here is a conservative estimate (upper bound) of the derivative $\pair{\op \pi'A_{\pi_k} ,\op\sigma_{\pi_k}^*\rho}$ within the class of $\pi'$ that is absolutely continuous with respect to $\pi_k$. This leads to a stable training algorithm even if the advantage function estimate is far from the true advantage function as long as $\pi_0$ has sufficiently large support. During training, the conservative estimate of the objective function ensures that the policy does not update drastically due to a bad estimate of the advantage function during sampling based reinforcement learning.

The convergence of a certain variation of the PPO algorithm is studied in \cite{lascu2025ppo} for MDPs with general state and action spaces.

\paragraph{Policy mirror descent} The regularized MDP problem was studied in \cite{geist2019theory}, where the authors added a strongly convex regularizer in the state-action value function. This idea was later generalized to a policy mirror descent algorithm in 
\cite{ju2022policy} with further developments of the algorithm and convergence rates in \cite{zhan2023policy}. Let  $D_\Pi$ be a Bregman divergence on the space of policies $\stable$ satisfying $D_\Pi(\pi,\pi')>0$ if $\pi(s)\neq \pi'(s)$ for a set of positive measure in the support of $\op \sigma_\pi^*\rho$. Let $D_{\mathcal P(\mathcal A)}$ be the restriction of the Bregman divergence on the distributions over the action set, and $\eta_k>0$ be a hyperparameter (stepsize). The policy mirror descent is defined as
\begin{align*}
    \pi_{k+1} &= \argmin_{\pi'\in\Pi}\pair{\op \pi'q_{\pi_k} + \frac{1}{\eta_k} \: D_\Pi(\pi_k,\pi'),\op\sigma_{\pi_k}^*\rho}\\
    \pi_{k+1}(s)&= \argmin_{p\in\mathcal P(\mathcal K(s))} \int q_{\pi_k}(s,a)p(da) + \frac{1}{\eta_k} \: D_{\mathcal P(\mathcal A)}(\pi_k(s),p).
\end{align*}
The objective vanishes at $\pi_k$; thus, $\pi_{k+1}$, if not equal to $\pi_k$ for almost every $s\in\supp(\op\sigma_{\pi_k}^*\rho)$, necessarily improves the performance:
\begin{align*}
    \op \pi_{k+1}q_{\pi_k} \leq - \frac{1}{\eta_k} \: D_\Pi(\pi_k,\pi_{k+1})<0 \text{ for $\op\sigma_{\pi_k}^*\rho$-almost every $s$}.
\end{align*}
If the space of policies is assumed to be the set of all measurable functions  $\pi:\mathcal S\to\mathcal P(\mathcal A)$ and $\mathcal K = \mathcal S\times\mathcal A$, then Assumption \ref{assm:interchange} is satisfied under fairly general conditions \cite{hiai1977integrals,bouchitte1988integral}. 

\subsection{Second Order and Approximate Second Order Policy Gradient Methods}
We now turn our attention to second order methods for solving for the optimal policy. 
\paragraph{Newton's method}
In Newton's method, the goal is to minimize the second order term. 
\begin{align*}
    \pi_{k+1} &= \argmin_{\pi'\in\Pi}  \pair{\op \pi'A_{\pi_k} + (\op\pi' - \op\pi
    _k)\op P\op \sigma_{\pi_k} \op\pi' A_{\pi_k},\op\sigma_{\pi_k}^*\rho} 
\end{align*}
This algorithm is infeasible to implement since the second term in the expression above is difficult to compute for a policy $\pi'$ that is unknown. Thus, one can use the majorization of the second term to arrive at the majorization mimimization algorithm as discussed below.

\paragraph{Trust region policy optimization}
In the trust region policy optimization \cite{schulman2015trust}, a trust region around the current policy is created and the policy advantage function is optimized within that trust region. In our case, this would be given by
\begin{align*}
    \pi_{k+1} &= \argmin_{\pi'\in\Pi} \langle L_{\pi_k}(\pi'),\rho\rangle \text{ s.t. } \metric(\pi_k,\pi')^2\leq \eta_k^2,
\end{align*}
where $\eta_k$ is a hyperparameter that is picked decreasingly at every time step $k$ of the algorithm. The algorithm above is a generalization of the original TRPO algorithm in \cite{schulman2015trust}.

\begin{remark}
To see the equivalence of the above algorithm with the TRPO algorithm \cite[Theorem 1]{schulman2015trust}, note that $\ws(s) \equiv 1$, $\wsa(s,a) \equiv 1$, $\lifs = \fs$, $\lifsa = \fsa$, $\mathcal K = \mathcal S\times\mathcal A$, and $\ipmas\ $ is the usual total variation norm. One can replace $\beta(q_{\pi_k},\pi')$ with its upper bound $\frac{4\gamma\|c\|_\infty}{(1-\gamma)^3}$ and $\ipmas\ (\pi_k(s),p)^2\leq D_{KL}(\pi_k(s)\| p)$, where $D_{KL}$ is the KL divergence. We further note here that while $\ipmas\ (\pi_k(s),p)^2\leq D_{KL}(p\|\pi_k(s))$, the paper \cite{schulman2015trust} took the other inequality in their TRPO algorithm.
\end{remark}

The paper \cite{schulman2015trust} assumed in their derivation that the reward is independent of the state in the finite MDP setting, which is not the case here. 


\paragraph{Majorization minimization (OTPG algorithm)} 
This algorithm is the basis of the TRPO algorithm discussed above. In this part, suppose that the hypotheses in Theorem \ref{thm:valueiteration}, Assumption \ref{assm:rhoPv} and Assumption \ref{assm:interchange} hold. Note that due to Theorem \ref{thm:majorization}, we have
\begin{align*}
    \pair{\op \pi'A_{\pi} + \gamma \op\Delta\op P\op \sigma_{\pi'} \op\Delta q_\pi,\op\sigma_{\pi_k}^*\rho} \leq \pair{\op \pi'A_{\pi} + \beta(q_\pi,\pi')\ws \metric(\pi',\pi)^2,\op\sigma_{\pi_k}^*\rho}
\end{align*}
At iteration $k$, assuming $\beta(q_{\pi_k},\pi')\leq \beta_k$ for some $\beta_k>0$ and for all $\pi'\in\stable$ (or at least, stable $\pi'$ in a small neighborhood around $\pi_k$). This leads to our majorization minimization algorithm, which is a generalization of the TRPO algorithm:
\begin{align}
    \pi_{k+1} &= \argmin_{\pi'\in\Pi}  \pair{\op \pi'A_{\pi_k} + \beta_k\ws \metric(\pi',\pi_k)^2,\op\sigma_{\pi_k}^*\rho} \nonumber\\
    \pi_{k+1}(s)&= \argmin_{p\in\mathcal P(\mathcal K(s))} \int A_{\pi_k}(s,a) p(da) +  \beta_k\ws(s) \ipmas\ (\pi_k(s),p)^2.\label{eqn:ipmppo}
\end{align}
We refer to this algorithm as OTPG (operator theoretic policy gradient) algorithm.

Through the right choice of integral probability metric such as maximum mean discrepancy in OTPG and by adding a KL divergence term to mimic mirror descent in the high dimensional policy space, we simplify the computation for finite MDPs, as demonstrated below.

\subsection{MM algorithm for Finite MDPs under RKHS}
Consider a finite MDP with $|\mathcal S| = n, |\mathcal A| = m$ and $\mathcal K=\mathcal S\times\mathcal A$. Let $Q\in \Re^{n\times n}$ and $R\in\Re^{m\times m}$ be two positive definite matrices. Let $\ws\equiv 1$ and $\wsa\equiv 1$ be the identity function. We define $\lifs$ and $\lifsa$ as the reproducing kernel Hilbert space (RKHS) with the kernel function $\kernels(s,s') = \frac{1}{2}Q_{ss'}$ and $\kernelsa((s,a),(s',a')) = \frac{1}{2}(Q_{ss'}+R_{aa'})$, where the generators $\gen{\lifs}$ and $\gen{\lifsa}$ are unit balls in the respective RKHS. The $\ipms$ and $\ipmsa$ are the usual maximum mean discrepancy (MMD). The existence of an optimal policy here is already well-known \cite{puterman2014markov} and the value iteration converges due to the contraction mapping theorem. 

With these definitions, $\mathcal P(\mathcal A)$ is the simplex in $m$ dimensions (we thus treat the probability measure as a vector) and $\ipmas\ (p_1,p_2)^2$ is given by $\ipmas\ (p_1,p_2)^2 = \frac{1}{2} \big(p_1^T R p_1 + p_2^T R p_2 - 2 p_1^T R p_2\big)$ for $p_1,p_2\in \mathcal P(\mathcal A)$. We apply \eqref{eqn:ipmppo} to derive a new TRPO algorithm:
\begin{align*}
    \pi_{k+1}(s)&= \argmin_{p\in\mathcal P(\mathcal A)} A_{\pi_k}(s,\cdot)^T p +  \beta_k \ipmas\ (\pi_k(s),p)^2.
\end{align*}
Note that since the MMD is bounded, we do not need the clipping function that is generally used in TRPO \cite{schulman2017proximal}. Given the quadratic structure of the objective function as a function of $p$ and the optimization over simplex, the above minimization problem can be solved approximately and iteratively using a mirror descent based algorithm  using entropic regularization \cite[p. 323]{bertsekas2016nonlinear} (see equation (3.115) and problem 3.6.3 in the book). In this case, we modify the objective function slightly using a hyperparamter $\eta_k\uparrow \infty$ yielding
\begin{align}
    \pi_{k+1}(s)&= \argmin_{p\in\mathcal P(\mathcal A)} A_{\pi_k}(s,\cdot)^T p +  \beta_k \ipmas\ (\pi_k(s),p)^2 + \frac{1}{\eta_k}D_{KL}(p\|\pi_k(s)). \label{eqn:newppo}
\end{align}
Note that the objective function on the right side of the equation above is a valid majorization on the performance difference as derived in Theorem \ref{thm:majorization} and is a strictly convex function of $p$. Thus, one can view it as a variant of TRPO algorithm as well. Let $R_a$ denote the $a^{th}$ column of the matrix $R$. Due to the problem 3.6.3 in \cite[p. 323]{bertsekas2016nonlinear}, \eqref{eqn:newppo} is solved using a low-complexity iterative scheme with $p_0 = \pi_k(s)$ and
\begin{align}
    \vartheta_{k,l}(s,a) &= -\eta_k \Big(A_{\pi_k}(s,a) - \beta_k R_a^T \big(\pi_k(s)-p_l\big)\Big)\\
    p_{l+1,a} &= \frac{1}{Z_k(p_l)} p_{l,a} \exp(\vartheta_{k,l}(s,a)),
    \text{ where } Z_k(p_l) = \sum_{a\in\mathcal A} p_{l,a} e^{\vartheta_{k,l}(s,a)}. \label{eqn:pimmrkhs}
\end{align}
By stopping the iteration early, say at iteration $\bar l$, one can set $\pi_{k+1}(s) = p_{\bar l}$ for all $s\in\mathcal S$ and proceed with computation of $A_{\pi_{k+1}}$. We refer to this update scheme as MM-RKHS algorithm to differentiate it with the PPO, TRPO, and OTPG algorithms.

Thus, the computation of the new policy using the advantage function of the old policy is straightforward. Note that by eliminating IPM and replacing the KL divergence term with $\beta D_{KL}(\pi_k(s)\| p)$ in \eqref{eqn:newppo}, we arrive at the usual TRPO algorithm as defined in equation (9) of \cite{schulman2015trust}. The above approach involving IPM, however, yields a class of TRPO algorithm depending on the choice of the positive definite matrix $R$. This has the potential to lead to faster convergence in comparison to the usual TRPO algorithm by carefully identifying the matrix $R$. Moreover, the above algorithm does not require computation of the second derivative of the KL divergence term; consequently, it is faster than TRPO in practice. 

One could also update $R$ as more training data is gathered to speed up the convergence. In sampling based reinforcement learning, some hyperparameter tuning for $(\beta_k,\eta_k)$ is also needed to ensure that the algorithm does not lead to rapidly changing policies and converges to a near-optimal solution due to bad estimates of the advantage function. Under certain further assumptions, the convergence of this algorithm for finite MDPs can be investigated using Proposition 7.3.5 in \cite{lange2016mm} and the gradient domination result from \cite[Lemma 4]{agarwal2021theory}.

\subsection{Numerical Simulations of MM-RKHS algorithm}
The goal of this section is to demonstrate the efficacy of the MM-RKHS algorithm defined above in comparison to widely used PPO algorithm. To do so, we study randomly generated  GARNET environments \cite{bhatnagar2009natural} with 1000 states, 200 actions, discount factor $\gamma = 0.95$, and the branching factor is 20 (that is, for each state-action pair, the state can transition to 20 other states). The initial distribution $\rho$ is uniform distribution over the entire state space $\mathcal S$. We consider the performance of the following algorithms.
\begin{enumerate}
    \item \textbf{PPO} where the cost and transition kernel are known and the advantage function is computed using closed form expression. The hyperparameter~$\epsilon$ in~\eqref{eq:PPO} is chosen as $\epsilon=0.2$ (which is the commonly recommended value \cite{schulman2017proximal}). For policy gradient algorithm, we used the learning rate of 0.8 and ran the policy update for 10 iterations using this learning rate.
    \item \textbf{Sample based PPO}, where the cost and transition kernel are not known and the advantage function is estimated from sampled trajectories. The hyperparameter are the same as that of PPO.
    \item  \textbf{MM-RKHS} with $R$ as the identity matrix, $\beta_k(s) = \|A_{\pi_k}(s,\cdot)\|_\infty /\sqrt{k+1}$, $\bar l = 1$, and $\eta_k = \eta_0(k+1)$ with $\eta_0 = 1.15$, where $\beta_k$ and $\eta_k$ were chosen to improve the stability of the algorithm. 
    \item \textbf{Sample based MM-RKHS}, where the cost and transition kernel are not known and the advantage function is estimated from sampled trajectories. In this settings, we clip the exponent $\vartheta_{k,l}(s,a)$ to be between $[-1.5,1.5]$ to improve the numerical stability, as is commonly done in reinforcement learning algorithms with advantage estimates. 
\end{enumerate}    

We run each algorithm for $K=50$ iterations. For each iteration for the sample-based algorithm, we generated 50000 samples per episode and 5 episodes to estimate the advantage function. The samples where generated independently for sample-based PPO and sample-based MM-RKHS algorithms. Advantage function was estimated using samples by employing the technique presented in \cite[Section 5.2]{sutton2018reinforcement}.

\begin{figure}
    \centering
    \includegraphics[width=0.48\linewidth]{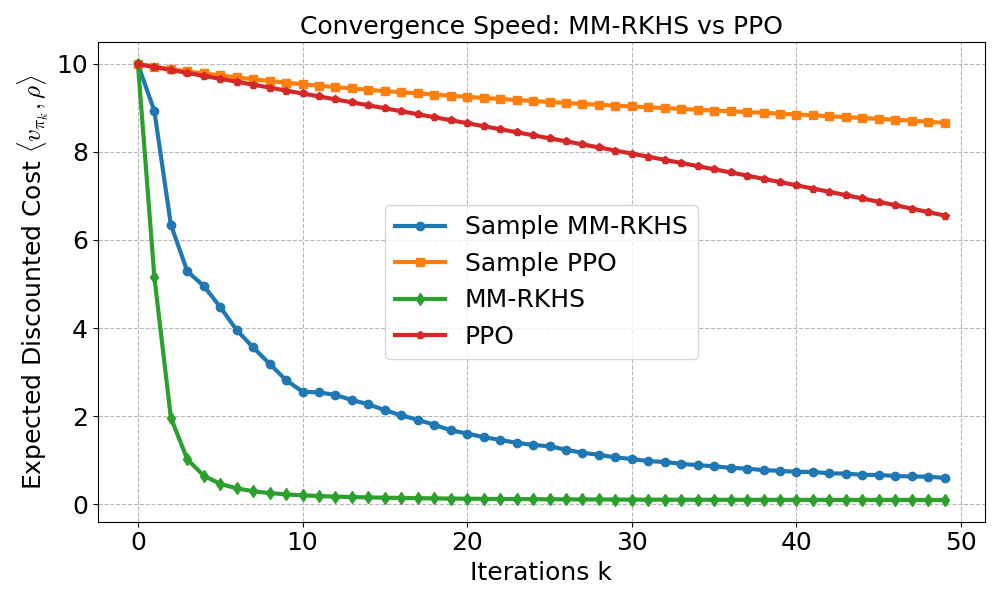}
    \includegraphics[width=0.48\linewidth]{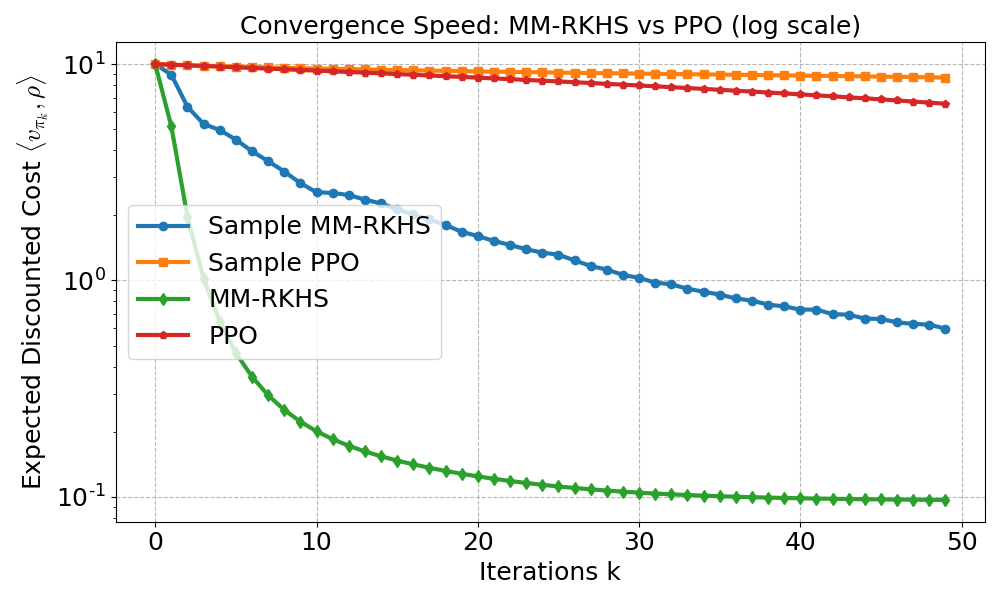}
    \caption{\label{fig:MMRKHS}MM-RKHS versus PPO on a randomly generated GARNET instance. Two variants of both algorithms are investigated: the standard variants compute the advantage function using closed-form formulas while the sample-based variants  compute the advantage function via sampled trajectories.}
\end{figure}

We show the result in Figure \ref{fig:MMRKHS}, where we plot $\pair{v_{\pi_k},\rho}$ and $\log(\pair{v_{\pi_k},\rho})$ on the y axis as a function of iteration index $k$ on x axis. We observe that there is a substantial speedup in sampling based learning using MM-RKHS algorithm. This is not surprising -- PPO approximates the TRPO algorithm, but largely uses first-order gradient update. On the other hand, the algorithm  MM-RKHS developed in this paper is minimizing a majorization of the second order term (just like TRPO algorithm). Consequently, it converges to the optimal solution much faster -- at a superlinear convergence rate as evident from the log scale graph. The sampling based MM-RKHS algorithm also enjoys a faster convergence speed as compared to sampling based PPO algorithm.

We admit the PPO has been implemented extensively across many applications, and is robust to hyperparameter selection (although it suffers from high sample complexity). Extensive simulations are needed for MM-RKHS algorithm as well to conclude that MM-RKHS is superior to PPO in terms of convergence speed and sample complexity across a variety of reinforcement learning applications. 

\section{Conclusion}\label{sec:conclude}
We investigated the problem of solving MDPs with general state and action spaces. By viewing the transition kernel and the policies as linear operators over certain spaces, we establish a new result for the existence of optimal policies in MDPs with general state and action spaces. This new framework then allows us to derive the policy difference lemma and majorization bound for computation of an optimal policy through majorization minimization algorithm. The algorithm thus obtained is termed as OTPG (operator theoretic policy gradient) algorithm. We further developed a variant of OTPG algorithm for finite MDPs, which we call MM-RKHS algorithm. We demonstrated through numerical simulation that MM-RKHS algorithm for finite MDPs is considerably faster and sample efficient than PPO algorithm on GARNET environment.

To improve the reinforcement learning algorithms for applicability in complex problems with infinite dimensional state/action space such as active flow control and new material discovery, further research needs to be done in which the transition kernel may have a certain structure such as it being generated from elliptical PDEs or has measure-preserving properties. Moreover, the cost function may have sparsity in these problems. Moreover, convergence properties of this class of reinforcement learning algorithms along with their gradient dominance properties \cite{agarwal2021theory} remain to be investigated. A general theory of reinforcement learning for solving inverse problems in scientific computation further remains an open problem. We believe that the theoretical framework developed here can help us solve these problems, and we leave them for the future research.

\appendix

\section{Proof of Theorem \ref{thm:existencevpi}}\label{app:existencevpi}

We prove the result separately for $\pi \in \specstable$ and $\pi \in \decay$.

\begin{lemma}\label{lem:pi-stable}
    If $\pi\in\specstable$ and $c_\pi\in\fs$, then $v_\pi\in\fs$ exists.
\end{lemma}
\begin{proof}
    By Theorem 10.13 on p. 235 of \cite{rudin1973functional}, if $\spec(\gamma\op P_\pi)<1$, then $\inf_{t\geq 0}\|\gamma^t\op P_\pi^t\|^{1/t}<1$; see Subsection \ref{sub:linearop} for a discussion on this. Let $t_0\in \Na$ be such that $\|\gamma^{t_0}\op P_\pi^{t_0}\|<1$. This implies that $(\ids -\gamma\op P_\pi)^{-1}$ is a bounded linear operator from Corollary 2.3.3 in \cite{atkinson2005theoretical}; consequently, $\sum_{t\geq 0} \gamma^t\op P_\pi^t = (\ids -\gamma\op P_\pi)^{-1}$. We use this fact to arrive at $v_\pi = \sum_{t\geq 0} \gamma^t\op P_\pi^t c_\pi = (\ids -\gamma\op P_\pi)^{-1}c_\pi$. This concludes the proof.
\end{proof}

\begin{lemma}\label{lem:pi-decay}
    If $\pi\in\decay$ and $c_\pi\in\fs$, then $v_\pi\in\fs$ exists. 
\end{lemma}
\begin{proof}
    The existence of $v_\pi$ is a direct application of the main result in \cite{suzuki1976convergence}.
\end{proof}

\section{Proof of Theorem \ref{thm:uniformlybounded}}\label{app:uniformlybounded}
Let $\mathcal X$ be a subspace of a Banach space and $\op A\in\mathcal B(\mathcal X,\mathcal X)$ be a bounded linear operator. Let $y\in\mathcal X$ be fixed and define an operator $\op B:\mathcal X\to\mathcal X$ as $\op B x = y+\op A x$. This yields $\op B^l x = \sum_{i=0}^{l-1} \op A^i y + \op A^l x$ for any $l\in\Na$. We have the following result. 
\begin{lemma}\label{lem:specA}
     If $\spec(\op A)<1$, then there exists constants $\tilde\kappa, \kappa>0$ such that for any $x\in\mathcal X$ satisfying $ \|x\|\leq \tilde\kappa\|y\|$, we have $\|\op B^k x\|\leq \kappa \|y\|$ for all $k\in \Na$.
\end{lemma}
\begin{proof}
If $\spec(\op A)<1$, then there exists $L\in\Na$ such that $\alpha:=\|\op A^L\|<1$ (see Subsection \ref{sub:linearop}).
Let $\beta = \sum_{i=0}^{L-1} \|\op A^i\|\in(0,\infty)$, $\tilde \kappa = \frac{\beta}{1-\alpha}>0$, and $\kappa = \beta+\tilde\kappa \max\{1,\|\op A\|,\ldots,\|\op A^{L-1}\|\}$. We now apply the principle of mathematical induction to complete the proof. First, we have an intermediate result: $\|\op B^{nL} x\| \leq \tilde \kappa \|y\|$ for all $n\in\Na$, which we establish through induction. Indeed, for $n=1$, if $ \|x\|\leq \tilde\kappa\|y\|$, then we have 
\[\|\op B^L x\|\leq \left\|\sum_{i=0}^{L-1} \op A^i y\right\| + \|\op A^L\| \|x\| \leq \beta \|y\| + \alpha \|x\|\leq \|\leq\beta \|y\| + \alpha \tilde \kappa \|y\| = \tilde \kappa \|y\|.\]
Suppose that this statement holds for $n\in\Na$. We use the same argument as above to get 
\begin{align*}
    \|\op B^{(n+1)L} x\| = \|\op B^L \op B^{nL}x\|\leq \left\|\sum_{i=0}^{L-1} \op A^i y\right\| + \|\op A^L\| \|\op  B^{nL} x\| \leq \tilde \kappa \|y\|. 
\end{align*}
The intermediate result in thus established. Let us establish the result now. For $1\leq l\leq L-1$, since $ \|x\|\leq \tilde \kappa\|y\|$, we have 
\begin{align}
    \|\op B^l x\| \leq \beta\|y\| + \|\op A^l\| \|x\|\leq (\beta+\|\op A^l\|\tilde\kappa)\|y\|\leq \kappa\|y\|.\label{eqn:Blx}
\end{align}
Next, we observe that $\op B^{nL+l} x = \op B^l \op B^{nL} x$, and we apply \eqref{eqn:Blx} with $x$ replaced with $\op B^{nL} x$ to conclude that $\|\op B^{nL+l} x\|\leq \kappa \|y\|$. This completes the induction step and we arrive at the conclusion that $\|\op B^k x\|\leq \kappa \|y\|$ for all $k\in \Na$ (recall that by construction, $\kappa>\tilde\kappa$). 
\end{proof}

Next, we note that since $c\geq 0$, we have $\op T v\leq \op T_\pi v$ for any $\pi\in\stable$. Moreover, the Bellman operator $\op T$ is monotone, that is, if $v_1\leq v_2$, then $\op T v_1\leq \op T v_2$. An application of the principle of mathematical induction implies that $\op T^k 0 \leq \op T_\pi^k 0$ for all $k\geq 0$ for any $\pi\in\stable$. 

 Pick any $\pi\in\stable$ and set $y = c_\pi$. If $\pi \in \specstable$, then pick $\op A = \gamma\op P_\pi$, $\mathcal X=\lifs$ and $\op B = \op T_\pi$. If $\pi \in \decay$, then pick $\op A = \gamma\op P_\pi\big|_{\mathcal H_\pi}$, $\mathcal X= \mathcal H_\pi$ and $\op B = \op T_\pi$, where $\mathcal H_\pi$ is defined in \eqref{eqn:H}. We can now apply Lemma \ref{lem:specA} to conclude that $\sup_{k\geq 0} \|\op T_\pi^k 0\|_{\fs} \leq \kappa \|c_\pi\|_{\fsa}<\infty$. This result, coupled with the fact that $\op T^k 0 \leq \op T_\pi^k 0$ for all $k\geq 0$, allows us to conclude the result. 


\section{Proof of Lemma \ref{lem:policydiff}}\label{app:policydiff}

We have the following lemma that delineates the property of the policy advantage function.
\begin{lemma}\label{lem:Deltaq}
    Let $\pi,\pi'\in\specstable$ and define $\op\Delta = \op\pi' - \op\pi$. We have $\op\Delta q_\pi = \op\Delta A_\pi = \op\pi' A_\pi$ and $L_{\pi}(\pi') =\op \sigma_\pi \op\Delta q_\pi $.
\end{lemma}
\begin{proof}
Note that $\op\pi A_\pi = 0$. Thus, $ \op \pi' A_\pi = \op \pi' A_\pi - \op \pi A_\pi = \op \pi'q_\pi - \op \pi q_\pi$. This yields
\begin{align*}
        L_{\pi}(\pi') & =\op \sigma_\pi \op \pi' A_\pi = \op \sigma_\pi (\op \pi'q_\pi - \op \pi q_\pi) =\op \sigma_\pi \op\Delta q_\pi = \op\sigma_\pi \Big(c_{\Delta} + \gamma \op P_\Delta \op\sigma_\pi c_\pi\Big),
    \end{align*}
    where $c_\Delta = \op\Delta c$ and $\op P_\Delta = \op\Delta \op P$.
\end{proof}

Now, we turn our attention to establishing Lemma \ref{lem:policydiff}. Since $\pi,\pi'\in\specstable$, $\op\sigma_{\pi}$ and $\op\sigma_{\pi'}$ are well-defined and $\op\sigma_{\pi'} = (\ids - \gamma \op P_\pi -\gamma \op P_\Delta)^{-1}$, where $\op P_\Delta = \op\Delta \op P$. We use \eqref{eqn:diffinvop2} to conclude that $\op\sigma_{\pi'} - \op\sigma_\pi = \gamma \op\sigma_{\pi'} \op P_\Delta \op\sigma_\pi = \gamma \op \sigma_\pi \op P_\Delta \op \sigma_{\pi} + \gamma^2 \op \sigma_\pi \op P_\Delta \op \sigma_{\pi'} \op P_\Delta \op \sigma_{\pi}$. Computing the expression of $v_{\pi'} - v_{\pi}$ and using \eqref{eqn:lpipi'} and \eqref{eqn:diffinvop2}, we get
\begin{align*}
    v_{\pi'} - v_{\pi} &=  \op \sigma_{\pi'} c_\pi - \op \sigma_{\pi} c_\pi + \op \sigma_{\pi'} c_{\pi'} - \op \sigma_{\pi'} c_\pi\\
    & = \gamma \op \sigma_\pi \op P_\Delta \op \sigma_{\pi}  c_{\pi} + \gamma^2 \op \sigma_\pi \op P_\Delta \op \sigma_{\pi'} \op P_\Delta \op \sigma_{\pi}  c_{\pi} + \op \sigma_{\pi}  c_\Delta + \gamma \op \sigma_{\pi}\op P_\Delta \op \sigma_{\pi'}   c_\Delta \\
    & =  L_{\pi}(\pi')   + \gamma  \op \sigma_\pi \op P_\Delta\op \sigma_{\pi'} ( c_\Delta + \gamma \op P_\Delta \op \sigma_{\pi}  c_{\pi}) \\
    & =  L_{\pi}(\pi')   + \gamma  \op \sigma_\pi \op\Delta\op P\op \sigma_{\pi'} \op\Delta q_\pi
\end{align*}
This completes the proof.
\newpage
\bibliographystyle{ieeetr}
\bibliography{math, ipm, rl, mdp, infinite, robotics}

\end{document}